\documentclass{article}
 
\usepackage{hyperref}       
\usepackage[acronym,smallcaps,nowarn,section,nogroupskip,nonumberlist]{glossaries}
\usepackage[nameinlink,capitalise]{cleveref}
\usepackage{amsthm}
\usepackage{amssymb}

\usepackage{url}            
\usepackage{booktabs}       
\usepackage{amsfonts}       
\usepackage{nicefrac}       
\usepackage{microtype}      
\usepackage[dvipsnames]{xcolor}

\usepackage{graphicx}
\usepackage{color}
\usepackage{todonotes}
\usepackage{enumitem}
\usepackage{caption}
\usepackage{wrapfig}
\usepackage{enumitem}

\usepackage{multirow}
\usepackage{makecell}

\usepackage[export]{adjustbox}
\usepackage{enumitem}

\DeclareMathOperator{\E}{\mathbb{E}}
\DeclareMathOperator{\R}{\mathcal{R}}
\DeclareMathOperator{\Cov}{\mathrm{Cov}}
\DeclareMathOperator{\Var}{\mathrm{Var}}

\DeclareMathOperator{\N}{\mathcal{N}}
\DeclareMathOperator{\Bern}{Bern}

\DeclareMathOperator{\vx}{\mathbf{x}}

\DeclareMathOperator{\vz}{\mathbf{z}}

\DeclareMathOperator{\vI}{\mathbf{I}}

\DeclareMathOperator{\vdz}{d\mathbf{z}}

\Crefname{algocf}{Algorithm}{Algorithms}
\crefname{algorithm}{Algorithm}{Algorithms}
\crefname{equation}{Equation}{Equations}
\crefname{figure}{Figure}{Figure}
\crefname{section}{§}{§§}
\Crefname{section}{§}{§§}
\crefformat{equation}{(#2#1#3)}

\newtheorem{theorem}{Theorem}

\newtheorem{lemma}{Lemma}

\usepackage[accepted]{icml2020}

\usepackage{caption}
\usepackage{subcaption}
\usepackage{mathtools,amssymb,amsmath,amsthm,amsfonts}
\usepackage{cancel}
\usepackage{ragged2e}
\usepackage{graphicx}
\usepackage{wrapfig}
\usepackage{lscape}
\usepackage{tikz}
\usepackage{environ}

\usepackage{times}
\usepackage{fancyhdr}
\usepackage[algo2e,ruled,vlined]{algorithm2e}
\SetKwInput{KwInput}{Input}
\SetKwInput{KwOutput}{Output}

\usepackage{listings}
\usepackage[dvipsnames]{xcolor}

\definecolor{codegreen}{rgb}{0,0.6,0}
\definecolor{codegray}{rgb}{0.5,0.5,0.5}
\definecolor{codepurple}{rgb}{0.58,0,0.82}
\definecolor{backcolour}{rgb}{0.95,0.95,0.92}

\lstdefinestyle{mystyle}{
    backgroundcolor=\color{backcolour},
    commentstyle=\color{codegreen},
    keywordstyle=\color{magenta},
    numberstyle=\tiny\color{codegray},
    stringstyle=\color{codepurple},
    basicstyle=\ttfamily\footnotesize,
    breakatwhitespace=false,
    breaklines=true,
    captionpos=b,
    keepspaces=true,
    numbers=left,
    numbersep=5pt,
    showspaces=false,
    showstringspaces=false,
    showtabs=false,
    tabsize=2
}
\DeclareMathOperator{\tpi}{\tilde{\pi}}
\DeclareMathOperator{\pib}{\pi_{\beta}(\vz|\vx)}
\DeclareMathOperator{\tpib}{\tilde{\pi}_{\beta}(\vx, \vz)}
\DeclareMathOperator{\ELBO}{\textsc{elbo}}
\DeclareMathOperator{\TVO}{\textsc{tvo}}
\DeclareMathOperator{\EUBO}{\textsc{eubo}}

\DeclareMathOperator{\MCMC}{\textsc{mcmc}}
\DeclareMathOperator{\REINFORCE}{\textsc{reinforce}}

\newcommand{\z}{\vz}
\newcommand{\x}{\vx}

\newcommand{\pxz}{{p_{\theta}(\x, \z)}}
\newcommand{\pzx}{{p_{\theta}(\z| \x)}}
\newcommand{\px}{{p_{\theta}(\x)}}

\newcommand{\qzx}{{q_{\phi}(\z| \x)}}

\newcommand{\qphi}{{q_{\phi}}}
\newcommand{\pibeta}{{\pi_{\beta}}}
\newcommand{\logiw}{{\log \frac{\pxz}{\qzx}}}

\newcommand{\dsymm}{{D_{KL}^{\, \,\leftrightarrow}}} 
\newtheorem{corollary}{Corollary}[theorem]

\newcommand{\tvolbb}{\TVO_{\textsc{L}}}
\newcommand{\tvolb}{\TVO_{L}(\theta, \phi, \x)}
\newcommand{\tvoub}{\TVO_{U}(\theta, \phi, \x)}
\newcommand{\tvolbfull}{\TVO_{\textsc{L}}(\theta, \phi, \x)}
\newcommand{\tvoubfull}{\TVO_{\textsc{U}}(\theta, \phi, \x)}

\usepackage{xpatch}
\xpatchcmd{\proof}{\itshape}{\normalfont\proofnamefont}{}{}

\newcommand{\proofnamefont}{\bfseries\itshape}

\glsdisablehyper{}
\newacronym{AIS}{ais}{annealed importance sampling}
\newacronym{BQ}{bq}{Bayesian Quadrature}
\newacronym{AUC}{auc}{area under the curve}
\newacronym{ELBO}{elbo}{evidence lower bound}
\newacronym{EUBO}{eubo}{evidence upper bound}
\newacronym{IS}{is}{importance sampling}
\newacronym{IWAE}{iwae}{importance weighted autoencoder}
\newacronym{KL}{kl}{Kullback-Leibler}
\newacronym{RWS}{rws}{reweighted wake-sleep}
\newacronym{SGD}{sgd}{stochastic gradient descent}
\newacronym{MCMC}{mcmc}{Markov Chain Monte Carlo}
\newacronym{SNIS}{snis}{self-normalized importance sampling}
\newacronym{TI}{ti}{thermodynamic integration}
\newacronym{TVI}{tvi}{thermodynamic variational inference}
\newacronym{TVO}{tvo}{thermodynamic variational objective}
\newacronym{VAE}{vae}{variational autoencoder}
\newacronym{VI}{vi}{variational inference}
\newacronym{VIMCO}{vimco}{variational inference for Monte Carlo objectives}
\newacronym{WS}{ws}{wake-sleep}

\definecolor{viterbired}{HTML}{990000}

\tikzset{->-/.style={decoration={
  markings,
  mark=at position .6 with {\pgftransformscale{0.75}\arrow{>}}},postaction={decorate}}}
\usepackage{tikz}
\usetikzlibrary{decorations}

\pgfdeclaredecoration{discontinuity}{start}{
  \state{start}[width=0.15\pgfdecoratedinputsegmentremainingdistance-0.5\pgfdecorationsegmentlength,next state=up from center]
  {}
  \state{up from center}[width=+.5\pgfdecorationsegmentlength, next state=big down]
  {
    \pgfpathlineto{\pgfpointorigin}
    \pgfpathlineto{\pgfqpoint{.25\pgfdecorationsegmentlength}{\4rpgfdecorationsegmentamplitude}}
  }
  \state{big down}[next state=center finish]
  {
    \pgfpathlineto{\pgfqpoint{.25\pgfdecorationsegmentlength}{-\pgfdecorationsegmentamplitude}}
  }
  \state{center finish}[width=0.5\pgfdecoratedinputsegmentremainingdistance, next state=do nothing]{
    \pgfpathlineto{\pgfpointorigin}
    \pgfpathlineto{{fig:symm_kl_one_piece}\pgfpointdecoratedinputsegmentlast}
  }
  \state{do nothing}[width=\pgfdecorationsegmentlength,next state=do nothing]{
    \pgfpathlineto{\pgfpointdecoratedinputsegmentlast}
  }
  \state{final}
  {
    \pgfpathlineto{\pgfpointdecoratedpathlast}
  }
}

\icmltitlerunning{Bregman Duality in Thermodynamic Variational Inference}

\begin{document}

\twocolumn[
\icmltitle{All in the Exponential Family: \\ Bregman Duality in Thermodynamic Variational Inference}



\icmlsetsymbol{equal}{*}

\begin{icmlauthorlist}
\icmlauthor{Rob Brekelmans}{isi,equal}
\icmlauthor{Vaden Masrani}{ubc,equal}
\icmlauthor{Frank Wood}{ubc}
\icmlauthor{Greg Ver Steeg}{isi}
\icmlauthor{Aram Galstyan}{isi}
\end{icmlauthorlist}

\icmlaffiliation{isi}{Information Sciences Institute, University of Southern California, Marina del Rey, CA}
\icmlaffiliation{ubc}{University of British Columbia, Vancouver, CA}

\icmlcorrespondingauthor{Rob Brekelmans}{brekelma@usc.edu}
\icmlcorrespondingauthor{Vaden Masrani}{vadmas@cs.ubc.ca}

\icmlkeywords{Machine Learning, ICML}

\vskip 0.3in
]
\printAffiliationsAndNotice{\icmlEqualContribution} 



\begin{abstract}
    The recently proposed Thermodynamic Variational Objective (\textsc{tvo}) leverages thermodynamic integration to provide a family of variational inference objectives, which both tighten and generalize the ubiquitous Evidence Lower Bound (\textsc{elbo}).  However, the tightness of \textsc{tvo} bounds was not previously known, an expensive grid search was used to choose a ``schedule'' of intermediate distributions, and model learning suffered with ostensibly tighter bounds.  In this work, we propose an exponential family interpretation of the geometric mixture curve underlying the \textsc{tvo} and various path sampling methods, which allows us to characterize the gap in \textsc{tvo} likelihood bounds as a sum of KL divergences.  We propose to choose intermediate distributions using equal spacing in the moment parameters of our exponential family, which matches grid search performance and allows the schedule to adaptively update over the course of training.  Finally, we derive a doubly reparameterized gradient estimator which improves model learning and allows the \textsc{tvo} to benefit from more refined bounds.  To further contextualize our contributions, we provide a unified framework for understanding thermodynamic integration and the \textsc{tvo} using Taylor series remainders.
\end{abstract}

\glsresetall
\allowdisplaybreaks

\section{Introduction}
\label{sec:introduction}

\begin{figure}[t]
\centering
\resizebox{0.45\textwidth}{!}{
\includegraphics{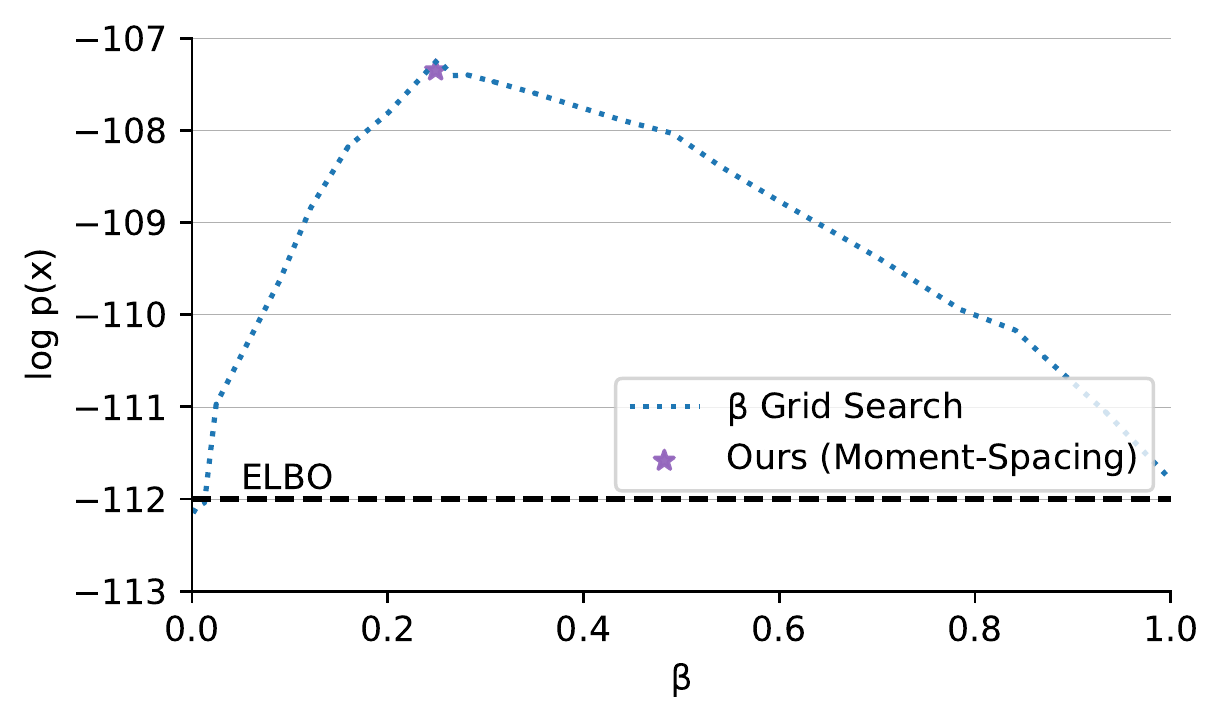}
}\vspace*{-.475cm}
\caption{\centering The original \gls{TVO} paper recommended using two partition points, with a single intermediate $\beta_1$ in addition to the \gls{ELBO} at $\beta_0=0$.  We report test $\log \px$ values from training a separate \gls{VAE} at each $\beta_1$, but this grid search is prohibitively expensive in practice.  Our moment-spacing schedule is an adaptive method for choosing $\beta$ points, which yields near-optimal performance on Omniglot and provides notable improvement over the \gls{ELBO}.}  \label{fig:single_beta_omni}  
\vspace*{-0.33cm}
\end{figure}

Modern  \gls{VI} techniques are able to jointly perform maximum likelihood parameter estimation and approximate posterior inference using stochastic gradient ascent~\cite{Kingma2013, Rezende2014}. 
Commonly, this is done by optimizing a tractable bound to the marginal log likelihood $\log p_{\theta}(\vx) = \int p_{\theta}(\vx, \vz)\vdz$, obtained by introducing a divergence $D[ \qzx || \pzx ]$ between the variational distribution $\qzx$ and true posterior $\pzx$~\cite{blei2017variational, li2016renyi, dieng2017variational, cremer2017reinterpreting, wang2018variational}. 

The recent Thermodynamic Variational Objective (\gls{TVO}) \cite{masrani2019thermodynamic} reframes likelihood estimation in terms of numerical integration along a geometric mixture path connecting $\qzx$ and $\pzx$.  This perspective yields a natural family of lower and upper bounds via Riemann sum approximations, with the \gls{ELBO} appearing as a single-term lower bound and \gls{WS} $\phi$ update corresponding to the simplest upper bound.  The \gls{TVO} generalizes these objectives by using a $K$-term Riemann sum to obtain tighter bounds on marginal likelihood.  We refer to the discrete partition $\{\beta_k\}_{k=0}^{K}$ used to construct this estimator as an `integration schedule.'

However, the gaps associated with these intermediate bounds was not previously known, an important roadblock to understanding the objective.  
Further, the \gls{TVO} was limited by a grid search procedure for choosing the integration schedule.   While  \gls{TVO} bounds should become tighter with more refined partitions, \citet{masrani2019thermodynamic} actually observe deteriorating performance in practice with high $K$.  

Our central contribution is an exponential family interpretation of the geometric mixture curve underlying the \gls{TVO} and various path sampling methods \cite{gelman1998simulating,neal2001annealed}.  Using the Bregman divergences associated with this family, we characterize the gaps in the \gls{TVO} upper and lower bounds as the sum of KL divergences along a given path, resolving this open question about the \gls{TVO}.   



Further, we propose to choose intermediate distributions in the \gls{TVO} based on the `moment-averaged' path of \citet{grosse2013annealing}, which arises naturally from the dual parameterization of our exponential family.  This scheduling scheme was originally proposed in the context of \gls{AIS}, where additional sampling procedures may be required to even approximate it.  We provide an efficient implementation for the \gls{TVO} setting, which allows the choice of $\beta$ to adapt to the shape of the integrand and degree of posterior mismatch throughout training. 

In Figure \ref{fig:single_beta_omni}, we observe that this flexible schedule yields near-optimal performance compared to grid search for a single intermediate distribution, so that the \gls{TVO} can significantly improve upon the \gls{ELBO} for minimal additional cost.  However, our moments scheduler can still suffer the previously observed performance degradation as the number of intermediate distributions increases.  As a final contribution, we propose a doubly reparameterized gradient estimator for the \gls{TVO}, which we show can avoid this undesirable behavior and improve overall performance in continuous models.




Our exponential family analysis may be of wider interest given the prevalence of \gls{MCMC} techniques utilizing geometric mixture paths \cite{neal1996sampling, neal2001annealed, grosse2016measuring, syed2019non, huang2020evaluating}.  To this end, we also present a framework for understanding \gls{TI} \cite{ogata1989monte} and the \gls{TVO} using Taylor series remainders, which clarifies that the \gls{TVO} is a first-order objective and provides geometric intuition for several results from \citet{grosse2013annealing}.  We hope these connections can help open new avenues for analysis at the intersection of \gls{MCMC}, \gls{VI}, and statistical physics.

\begin{figure*}[t]
    \begin{minipage}{\textwidth}
    \includegraphics[width=\textwidth]{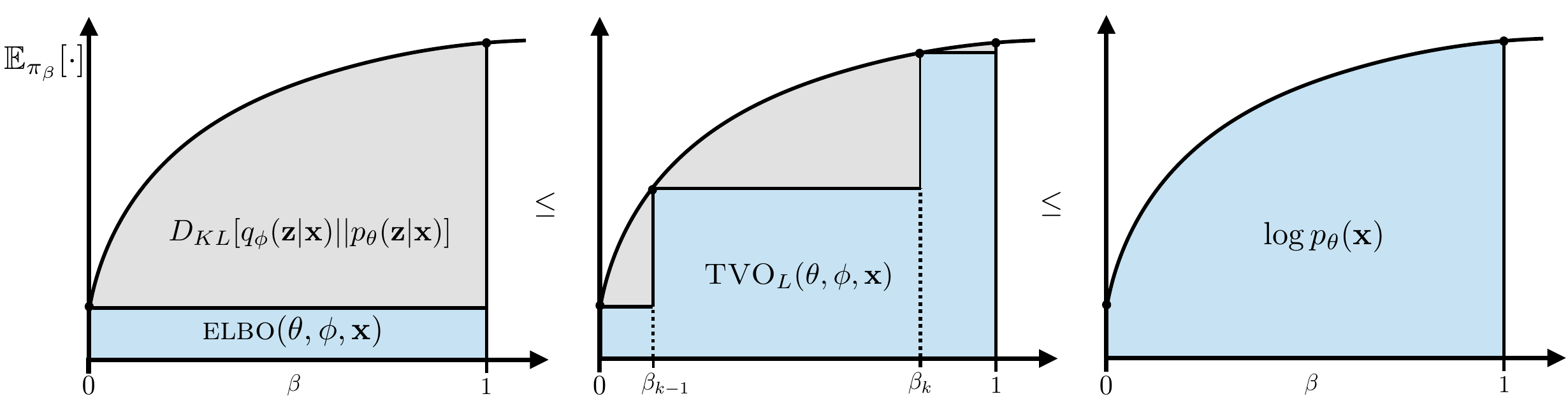}
    \captionof{figure}{\centering The \gls{TVO} is a K-term Riemann sum approximation of $\log\px$, which can be expressed as a scalar integral over the unit interval in \eqref{eq:tvi} and on the right. The $\ELBO$ is a single-term left Riemann approximation of the same integral using the point $\beta=0$ with $\pi_0 = \qzx$. Note that the integrand is negative in practice, but shown as positive for interpretability.} \label{fig:tvo_elbo_logpx} 
    \end{minipage}%
\end{figure*}

\section{Background}\label{sec:background}
\subsection{Thermodynamic Integration} 
Thermodynamic integration (\gls{TI}) is a technique from statistical physics, which frames estimating ratios of partition functions as a one-dimensional integration problem.  Commonly, this integral is taken over $\beta \in [0,1]$, which parameterizes a path of geometric mixtures between a base distribution $\pi_0$, and a target distribution $\pi_1$ \cite{gelman1998simulating}  
\vspace*{0.1cm}
\begin{align}
    \pi_{\beta}(\z) := \frac{\tilde{\pi}_\beta(\z)}{\int \tilde{\pi}_\beta(\z) \, d\z} = \frac{\pi_0^{1-\beta}(\z) \pi_1^{\beta}(\z) }{Z_{\beta}} \,. \label{eq:geomix_ti}
\end{align}
The insight of \gls{TI} is to recognize that, while the log partition function is intractable, its derivative can be written as an expectation that may be estimated using sampling or simulation techniques  \cite{neal2001annealed,habeck2017model}
\begin{align}
   \nabla_{\beta} \log Z_{\beta} &= \E_{\pibeta} \left[\log \frac{\pi_1(\z)}{\pi_0(\z)} \right] \,. \label{eq:equivalence}
\end{align}
In Sec. \ref{sec:expfamily}, we will see that this identity arises from an interpretation of the geometric mixture curve \cref{eq:geomix_ti} as an exponential family.  Applying this within the fundamental theorem of calculus, 
\begin{align}
    \log Z_1 - \log Z_0 &= \int_0^1 \nabla_{\beta} \log Z_{\beta} \, d\beta \label{eq:integral_zb} \\
                        &= \int_0^1 \E_{\pibeta} \left[\log \frac{\pi_1(\z)}{\pi_0(\z)} \right]d\beta. \label{eq:integral}
\end{align}
While \eqref{eq:integral_zb} holds for any choice of path parameterized by $\beta$, we can construct efficient estimators of the integrand in \eqref{eq:integral} and estimate the partition function ratio $\log Z_1/Z_0$ using numerical integration techniques.

\subsection{The Thermodynamic Variational Objective}\label{sec:tvo}
The \gls{TVO} \cite{masrani2019thermodynamic} uses \gls{TI} in the context of variational inference to provide natural upper and lower bounds on the log evidence, which can then be used as objectives for training latent variable models.
In particular, the geometric mixture path interpolates between the approximate posterior $\qzx$ and the joint generative model $\pxz$
\begin{align}
    \hspace*{-.2cm} \pib &= \frac{\tpib}{\int \tpib \vdz} := \frac{\qzx^{1-\beta} \pxz^{\beta}}{Z_{\beta}(\vx)} \, . \label{eq:geomix}
\end{align}
As distributions over $\z$, we can identify the endpoints as $\pi_0(\z|\x)=\qzx$ and $\pi_1(\z|\x)=\pzx$, with corresponding normalizing constants $Z_{0}=1$ and $Z_1 = \int \pxz \, d\z = \px$.

Applying \gls{TI} \cref{eq:integral_zb} for this set of log partition functions, \citet{masrani2019thermodynamic} express the generative model likelihood using a one-dimensional integral over the unit interval
\begin{align}
    \log\px &= \int_0^1 \E_{\pibeta} \left[\logiw \right]d\beta. \label{eq:tvi}
\end{align} 
The left and right endpoints of this integrand correspond to familiar lower and upper bounds on $\log \px$. The evidence lower bound (\gls{ELBO}) occurs at $\beta=0$, while the analogous \gls{EUBO} at $\beta = 1$ uses the `reverse' KL divergence and appears in various wake-sleep objectives (\gls{WS})  ~\citep{hinton1995wake, Bornschein2014} 
\begin{align}
    \ELBO(\theta, \phi, \x) &= \log \px - D_{KL}[q_\phi||p_\theta] \label{eq:elbo} \\
     \EUBO(\theta, \phi, \x) &= \log \px + D_{KL}[p_\theta||q_\phi] \label{eq:eubo}.
 \end{align}
To arrive at the \gls{TVO}, a discrete partition schedule is chosen $\mathcal{P}_{\beta} = \{ \beta_k \}_{k=0}^{K}$ with $\beta_0 = 0$, $\beta_K = 1$, and $\Delta_{\beta_k} = \beta_k- \beta_{k-1}.$  The integral in \cref{eq:tvi} is then approximated using a left or right Riemann sum.  Since \citet{masrani2019thermodynamic} show the integrand is increasing, these approximations yield valid lower and upper bounds on the marginal likelihood
\begin{align}
    \tvolbfull &:= \sum \limits_{k=1}^{K} \Delta_{\beta_k} \E_{\pi_{\beta_{k-1}}} \bigg[ \logiw \bigg] \label{eq:tvo_l} \\
    \tvoubfull &:= \sum \limits_{k=1}^{K} \Delta_{\beta_k} \E_{\pi_{\beta_{k}}} \bigg[ \logiw \bigg] \label{eq:tvo_u}
\end{align}
with
\begin{align}
    \tvolbfull \leq \log\px \leq \tvoubfull. \label{eq:full_bounds}
\end{align}
The first term of $\tvolb$ corresponds to the $\ELBO$, while the last term of $\tvoub$ corresponds to the $\EUBO$. Thus, the $\TVO$ generalizes both objectives, with additional partitions leading to tighter bounds on likelihood as visualized in Fig. ~\ref{fig:tvo_elbo_logpx}.

Although we consider thermodynamic integration over $0 \leq \beta \leq 1$ to approximate $\log \px$, note that this integral does not avoid the need for integration over $\z$ since each intermediate distribution must be normalized. \citet{masrani2019thermodynamic} propose an efficient, \gls{SNIS} scheme with proposal $\qzx$, so that expectations at any intermediate $\beta$ can be estimated by simply reweighting a single set of importance samples
\begin{align}
\mathbb{E}_{\pibeta} [\cdot] &\approx \sum \limits_{i=1}^{S} \frac{w_i^{\beta}}{\sum_{i=1}^{S} w_s^{\beta}} \, [ \cdot ] \, \, \text{where } w_i := \frac{p_{\theta}(\vx, \vz_i)}{q_{\phi}(\vz_i|\vx)}. \label{eq:snis}
\end{align}

\section{Exponential Family Interpretation}\label{sec:expfamily}
We propose a novel exponential family of distributions which, by absorbing both $\pxz$ and $\qzx$ into the sufficient statistic, corresponds to the geometric mixture path defined in~\cref{eq:geomix}.
We provide a formal definition in Sec. \ref{sec:expfamily_definition}, before showing in Sec. \ref{sec:expfamily_sanity_checks} that several key quantities in the \gls{TVO} arise from familiar properties of exponential families.  In Sec. \ref{sec:bregman_tvo}, we leverage the Bregman divergences associated with our exponential family to naturally characterize the gap in \gls{TVO} bounds as a sum of KL divergences.
\subsection{Definition}
\label{sec:expfamily_definition}
To match the \gls{TVO} setting in~\cref{eq:geomix}, we consider an exponential family of distributions with natural parameter $\beta$, base measure $\qzx$, and sufficient statistics equal to the log importance weights as in \cref{eq:tvo_l}-\cref{eq:tvo_u}
\begin{align}
    & \pibeta(\z|\x) := \pi_0(\z|\x) \exp \{ \, \beta  \cdot  T(\vx,\z) - \psi(\vx; \beta) \}  \label{eq:exp_fam} \\[1.25ex]
    & \text{where} \, \, T(\x, \z) := \logiw  \quad  \pi_0(\z|\x) := \qzx \nonumber
    \end{align}
This induces a log-partition function $\psi(\vx; \beta)$, which normalizes over $\z$ and corresponds to $\log Z_{\beta}(\vx)$ in~\cref{eq:geomix}
\begin{align}
    \psi(\vx; \beta ) &:= \log \int \qzx \exp \{ \, \beta \, \logiw\} \vdz, \nonumber \\
    &= \log \int \qzx^{1-\beta} \pxz^{\beta} \vdz \label{eq:partition_renyi} \\
    &= \log Z_{\beta}(\vx). \label{eq:log_zb}
\end{align}
The log-partition function will play a key role in our analysis, often written as $\psi(\beta)$ to omit the dependence on $\vx$.

We emphasize that we have made no additional assumptions on $\pxz$ or $\qzx$, and do not assume they come from exponential families themselves.   This `higher-order' exponential family thus maintains full generality and may be constructed between arbitrary distributions.

\subsection{TVO using Exponential Families}
\label{sec:expfamily_sanity_checks}
We now show that a number of key quantities, which were manually derived in the original \gls{TVO} work, may be directly obtained from our exponential family.
\paragraph{TI Integrates the Mean Parameters }It is well known that the log-partition function $\psi(\beta)$ is convex, with its first (partial) derivative equal to the expectation of the sufficient statistics under $\pi_{\beta}$ \cite{wainwrightjordan}.
\begin{align}
    \hspace*{-.2cm} \eta_{\beta} := \nabla_\beta \psi(\beta) = \E \, [T(\x,\z)] = \E_{\pibeta} \bigg[ \logiw \bigg] \label{eq:deriv}
\end{align}
This quantity is known as the \textit{mean parameter} $\eta_{\beta}$, which provides a dual coordinate system for indexing intermediate distributions (\citet{wainwrightjordan} Sec. 3.5.2). Comparing with \cref{eq:equivalence} and \cref{eq:tvi}, we observe that the ability to trade derivatives of the log-partition function for expectations in \gls{TI} arises from this property of exponential families.

We may then interpret the \gls{TVO} as integrating over the mean parameters $\eta_{\beta}= \nabla_{\beta} \log Z_{\beta}$ of our path exponential family, which can be seen by rewriting \cref{eq:integral_zb}
\begin{align}
    \psi(1) - \psi(0) &= \int_0^1 \eta_{\beta} \, d\beta = \int_0^1 \E_{\pibeta} \bigg[ \logiw \bigg] d\beta.  \nonumber 
\end{align}

\paragraph{TVO Likelihood Bounds}  The convexity of the log partition function arises from the fact that entries in its matrix of second partial derivatives with respect to the natural parameters correspond to the (co)variance of the sufficient statistics \cite{wainwrightjordan}.  In our 1-d case, this corresponds to the variance of the log importance weights
\begin{align}
 \nabla^2_\beta \psi(\beta) = \Var[T(\x,\z)] = \Var_{\pibeta} \bigg[ \logiw \bigg]. \label{eq:hess}
\end{align}
We can see that the \gls{TVO} integrand $\nabla_\beta \psi(\beta)$ is increasing from non-negativity of $\nabla^2_\beta \psi(\beta) \geq 0 \,\, \forall \beta$, which ensures that the left and right Riemann sums will yield valid lower and upper bounds on the marginal log likelihood.

\paragraph{ELBO on the Graph of $\mathbf{\eta_{\beta}}$ } Inspecting Fig. \ref{fig:tvo_elbo_logpx}, we see that the gap in the \gls{TVO} bounds corresponds to the amount by which a Riemann approximation under- or over-estimates the \gls{AUC}.  We can solidify this intution for the case of the \gls{ELBO}, a single-term approximation of $\log \px$ using $\beta = 0$ for the entire interval $\beta_1 - \beta_0 = 1-0$
\begin{align} 
\textsc{gap} &= \bigg[ \underbrace{\int_0^1 \nabla_\beta \psi(\beta) d\beta}_{\gls{AUC}} \bigg] - \underbrace{(1-0)\vphantom{\logiw}}_{\textsc{width}} \underbrace{\mathbb{E}_{\pi_0} \big[ \logiw \big]}_{\textsc{height}} \nonumber \\[1.25ex]
 &= \log\px - \ELBO(\theta, \phi, \x) \nonumber \\[1.25ex]
 &=  D_{KL}[q_{\phi}(\z|\x)||p_{\theta}(\z|\x)]. \label{eq:elbo_gap}
\end{align}
In the next section, we generalize this reasoning to more refined partitions, showing that the gap in arbitrary \gls{TVO} bounds corresponds to a sum of KL divergences between adjacent $\pi_{\beta_k}$ along a given path $\{\beta_k\}_{k=0}^{K}$.
\begin{figure*}[t]
    \begin{minipage}{.4875\textwidth}
    \vspace*{.2cm}
    \includegraphics[width=.9\textwidth]{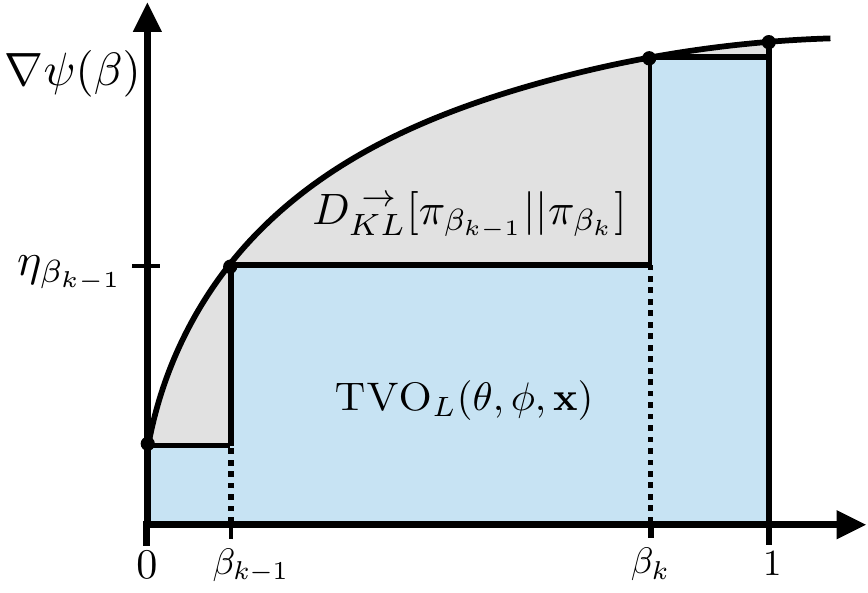}
    \captionof{figure}{The Bregman divergence $D_{KL}[\pi_{\beta_{k-1}} || \pi_{\beta_k} ]$ can be visualized as the area under the curve minus the left-Riemann sum via \cref{eq:gap_1}.  This term contributes to the gap in the likelihood bound $\tvolbb$.  We also derive an integral form for the KL divergence in App. \ref{app:taylor_kl}.  Note that both the integrand and $\psi(\beta)$ are negative in practice.} \label{fig:kl_left} 
    \end{minipage}
    \hspace{0.025\textwidth}
    \begin{minipage}{.4875\textwidth}
        \vspace*{.2cm}
        \includegraphics[width=.9\textwidth]{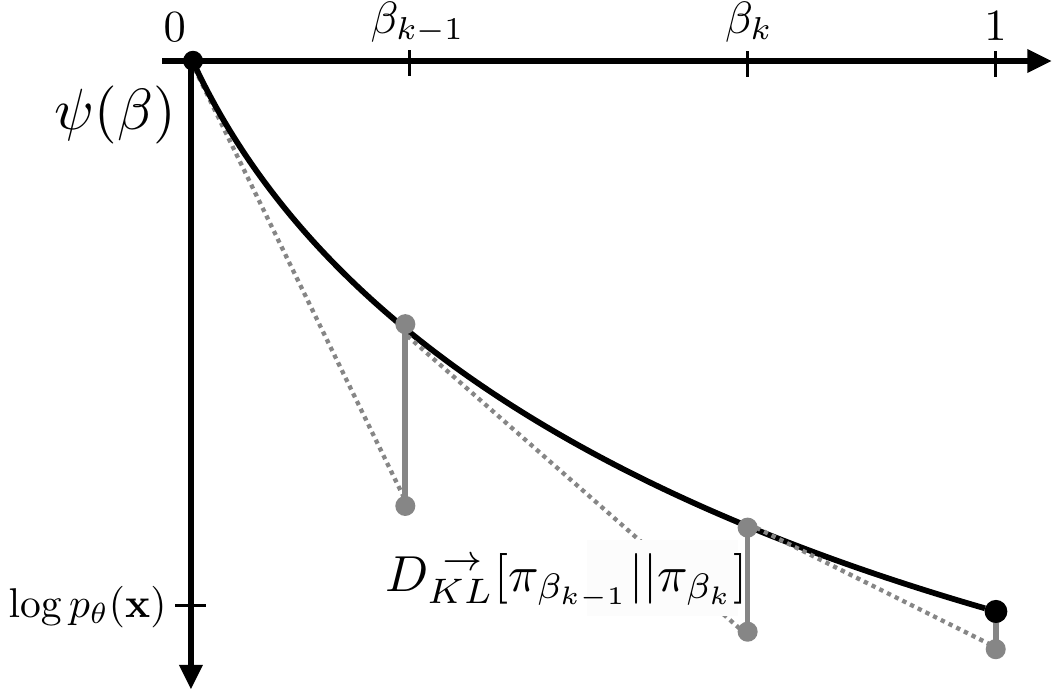} 
        \vspace*{.25cm}
    \captionof{figure}{$\tvolb$ may be viewed as constructing successive first-order Taylor approximations to intermediate $\psi(\beta_k)$, with the accumulated error corresponding to the gap in the bound.  The upper bound takes KL divergences in the reverse direction, with the first argument decreasing along the path.} \label{fig:kl_taylor}
\end{minipage}
\end{figure*}

\section{TVO Likelihood Bound Gaps} \label{sec:bregman_tvo}
In previous work, it was shown only that $\tvolb$ minimizes a quantity that is non-negative and vanishes at $\qzx = \pzx$ \cite{masrani2019thermodynamic}.   Using the Bregman divergences associated with our path exponential family, we can now provide a unified characterization of the gaps in \gls{TVO} bounds.

\subsection{Bregman Divergences} \label{sec:bregman} 
We begin with a brief review of the Bregman divergence, which can be visualized on the graph of the \gls{TVO} integrand in Fig. \ref{fig:kl_left} or  the log partition function in Fig. \ref{fig:kl_taylor}.

A Bregman divergence $D_{\psi}$ is defined with respect to a convex function $\psi$ ~\cite{Banerjee2005} which, in our case, takes distributions indexed by natural parameters $\beta$ and $\beta^{\prime}$ as its arguments
\begin{align}
     \hspace*{-.2cm} D_{\psi}[\beta : \beta^{\prime}] = \psi(\beta) - \big(\underbrace{\psi(\beta^{\prime}) + (\beta - \beta^{\prime}) \nabla_{\beta} \psi(\beta^{\prime})}_{\text{First Order Taylor Approx}}\big). \label{eq:breg_def}
\end{align}
Geometrically, the Bregman divergence corresponds to the gap in a first-order Taylor approximation of $\psi(\beta)$ around the second argument $\beta^{\prime}$, as depicted in Fig. \ref{fig:kl_taylor}.  Note that this difference is guaranteed to be nonnegative, since we know that the tangent will everywhere underestimate a convex function \cite{boyd2004convex}.

The Bregman divergence $D_{\psi}$ for the exponential family in \eqref{eq:exp_fam} is also equivalent to the KL divergence, with the order of the arguments reversed (also see  App. \ref{app:duality}).  Applying \eqref{eq:deriv} and adding and subtracting a base measure term,
\begin{align}
    D_{\psi}[\beta:\beta^{\prime}] &= \psi(\beta) - \psi(\beta^{\prime}) - (\beta - \beta^{\prime}) \nabla_{\beta} \psi(\beta^{\prime}) \\[1.15ex]
    &= \psi(\beta) \, - \beta\phantom{^{\prime}} \cdot \mathbb{E}_{\pi_{\beta^{\prime}}}[ \, T \, ]  - \mathbb{E}_{\pi_{\beta^{\prime}}}[ \, \log \pi_0 \, ]  \\
    & - \psi(\beta^{\prime}) + \beta^{\prime} \cdot \mathbb{E}_{\pi_{\beta^{\prime}}}[ \, T \,] + \mathbb{E}_{\pi_{\beta^{\prime}}}[ \, \log \pi_0 \, ] \nonumber \\
    &= \mathbb{E}_{\pi_{\beta^{\prime}}} [ \, \log \pi_{\beta^{\prime}} - \log \pi_{\beta} \, ], 
\end{align}
where in the third line, we use the fact that $\mathbb{E} \, [ \log \pi_{\beta}] = \mathbb{E} \, [\log \pi_0(\vz|\vx) + \beta \cdot T(\vx,\vz)] - \psi(\vx; \beta)$ from \eqref{eq:exp_fam}.  We then obtain our desired result, with
\begin{align}
    D_{\psi}[\beta:\beta^{\prime}] &= D_{KL}[\,  \pi_{\beta^{\prime}} \, || \pi_{\beta} \, ] \, .\label{eq:kl_psi}
\end{align}
\paragraph{KL Divergence on the Graph of $\eta_{\beta}$}
We can also visualize the Bregman divergence on the graph of the integrand $\eta_{\beta} = \nabla_{\beta}\psi(\beta)$ in Fig. \ref{fig:kl_left}, which leads to a natural expression for the gaps in \gls{TVO} upper and lower bounds.

To begin, we consider a single subinterval $[\beta_{k-1}, \beta_k]$ and follow the same reasoning as for the \gls{ELBO} in Sec. \ref{sec:expfamily_sanity_checks}.  In particular, the area under the integrand in this region is $\textsc{auc} = \int_{\beta_{k-1}}^{\beta_{k}} \nabla_\beta \psi(\beta) d\beta = \psi(\beta_k) - \psi(\beta_{k-1})$, with the left-Riemann approximation corresponding to $(\beta_k - \beta_{k-1}) \nabla_{\beta} \psi(\beta_k)$.  Taking the difference between these expressions, we obtain the definition of the Bregman divergence in \eqref{eq:breg_def}
\begin{align}
     \textsc{gap} &= \underbrace{\psi(\beta_{k})-\psi(\beta_{k-1})}_{\textsc{auc}} - \underbrace{(\beta_{k} - \beta_{k-1}) \nabla_{\beta} \psi(\beta_{k-1})}_{\text{Term in } \tvolb} \nonumber \\[1.75ex]
     &= D_{\psi}[\beta_{k}:\beta_{k-1}] = D^{\, \rightarrow}_{KL}[\pi_{\beta_{k-1}} || \pi_{\beta_{k}}]. \label{eq:gap_1}
\end{align}
where arrows indicate whether the first argument of the KL divergence is increasing or decreasing along the path.
For the gap in the right-Riemann upper bound, we follow similar derivations with the order of the arguments reversed in Sec.~\ref{sec:ubgap}.  This results in a gap of $ D_{KL}^{\leftarrow}[\pi_{\beta_{k}} || \pi_{\beta_{k-1}}]$, with expectations $\eta_{\beta_k}=\nabla_{\beta} \psi(\beta_{k})$ taken under $\pi_{\beta_k}$.

\subsection{TVO Lower Bound Gap} \label{sec:bregman_tvo_all_partition}%
Extending the above reasoning to the entire unit interval, we can consider any sorted partition $\mathcal{P}_{\beta} = \{ \beta_k \}_{k=0}^{K}$ with $\beta_0 = 0$ and $\beta_K = 1$.  Summing \eqref{eq:gap_1} across intervals, note that intermediate $\psi(\beta_k)$ terms cancel in telescoping fashion 
\vspace*{-.2cm}
\begin{align}
&\sum_{k=1}^{K} D_{\psi}[\beta_{k}:\beta_{k-1}] \label{eq:breg_sum} \\
&\phantom{\sum_{k=1}^{K}} = \psi(1) - \psi(0) - \sum_{k=1}^{K} (\beta_k - \beta_{k-1}) \nabla_{\beta} \psi(\beta_{k-1}) \, \nonumber
\end{align}
where the last term matches the $\tvolbb$ objective in ~\cref{eq:tvo_l}.

Writing $D_{\psi}$ as a KL divergence as in \eqref{eq:kl_psi} and recalling that $\psi(1)-\psi(0) = \log \px$, we obtain
\begin{align}
    \log p(\x) - \tvolb = \sum \limits_{k=1}^{K} D_{KL}^{\,\rightarrow}[\pi_{\beta_{k-1}} || \pi_{\beta_{k}}].  \label{eq:breg_tvo_lb}
\end{align}
We therefore see that the gap in the \gls{TVO} lower bound is the sum of KL divergences between adjacent $\pi_{\beta_{k}}$ distributions.

Alternatively, we can view \eqref{eq:breg_sum} as constructing successive first-order Taylor approximations to intermediate $\psi(\beta_k)$ in Fig. \ref{fig:kl_taylor}.  The likelihood bound gap of $\sum_{k=1}^{K} D_{\psi}[\beta_{k}:\beta_{k-1}]$ measures the accumulated error along the path.  While the \gls{ELBO} estimates $\psi(1) = \log \px$ directly from $\beta = 0$, more refined partitions can reduce the error and improve our bounds.  As $K \rightarrow \infty$, $\tvolb$ becomes tight as our $\pi_{\beta_k}$ are infinitesimally close, and the Riemann integral estimate would become exact given exact estimates of $\eta_{\beta_k}$.

\subsection{TVO Upper Bound Gap}\label{sec:ubgap}
To characterize the gap in the upper bound, we first leverage convex duality to obtain a Bregman divergence in terms of the conjugate function $\psi^{*}(\eta)$ and the mean parameters $\eta$.  As shown in App. \ref{app:duality}, this divergence,  $D_{\psi^{*}}$, is equivalent to $D_{\psi}$ with the order of arguments reversed 
\begin{align}
    D_{\psi^{*}}[\eta_k:\eta_{k-1}]&=~D_{\psi}[\beta_{k-1}~:~\beta_k] \label{eq:gap_2}  \\
    &=~D_{KL}^{\,\leftarrow}[ \, \pi_{\beta_k} || \, \pi_{\beta_{k-1}}] .  \nonumber
\end{align}
As in \cref{eq:breg_sum}, we expand the dual divergences along a path as 
\vspace*{-.125cm}
\begin{align}
    &\sum_{k=1}^{K} D_{\psi}[\beta_{k-1}:\beta_{k}] \\[0.25ex]
    &\phantom{\sum_{k=1}^{K}}= \psi(0) - \psi(1) - \sum_{k=1}^{K} (\beta_{k-1} - \beta_{k}) \nabla_{\beta} \psi(\beta_{k}) \nonumber
\end{align}
Since the last term corresponds to a right-Riemann sum, we can similarly characterize the gap in $\tvoub$ using a sum of KL divergences in the reverse direction
\begin{align}
\tvoub - \log p(\x) &= \sum \limits_{k=1}^{K} D_{KL}^{\,\leftarrow}[ \pi_{\beta_{k}} || \pi_{\beta_{k-1}}].
\end{align}

\subsection{Integral Forms and Symmetrized KL}\label{sec:taylor4}
To further contextualize the developments in this section, we show in App. \ref{app:taylor} that both thermodynamic integration and the \gls{TVO} may be understood using the integral form of the Taylor remainder theorem.  In particular, the expression \cref{eq:integral} underlying \gls{TI}  corresponds to the gap in a zero-order approximation, whereas we have previously shown that the KL divergence arises from a first-order remainder.

We can thus obtain integral expressions for the KL divergence, lending further intution for its interpretation as the area of a region in Fig. \ref{fig:kl_left} or Fig. \ref{fig:symm_kl_one_piece}
\begin{align}
    D_{KL}^{\, \,\rightarrow}[\pi_{\beta_{k-1}} || \pi_{\beta_k}] = \int \limits_{\beta_{k-1}}^{\beta_k}(\beta_k - \beta) \Var_{\pibeta} \big[ \logiw \big] d\beta. \nonumber
\end{align}
Combining the remainders in each direction, we recover a known identity relating the symmetrized KL divergence to the integral of the Fisher information (App. \ref{app:symm_kl} \eqref{eq:symm_kl_integralA}, \citet{dabak2002relations}).  

Similarly, we can visualize (twice) the symmetrized KL as the area of a rectangle in Fig. \ref{fig:symm_kl_one_piece}, by adding the gaps in the left- and right-Riemann approximations for a single interval
\begin{align}
\dsymm[ \pi_{\beta_{k-1}} ; \pi_{\beta_{k}} ] &=  D_{KL}^{\,\rightarrow}[ \pi_{\beta_{k-1}} || \pi_{\beta_{k}}] + D_{KL}^{\,\leftarrow}[ \pi_{\beta_{k}} || \pi_{\beta_{k-1}}] \nonumber \\
&= (\beta_k - \beta_{k-1})(\eta_{k} - \eta_{k-1}) \, .\label{eq:symm_kl_square}
\end{align}
From the Taylor remainder perspective, we note that \eqref{eq:symm_kl_square} can be derived using a further application of \gls{TI}, or the fundamental theorem of calculus, to the function $\nabla_{\beta} \psi(\beta)$, with $~\eta_k-\eta_{k-1}~= \int_{\beta_{k-1}}^{\beta_k} \nabla^2_{\beta} \psi(\beta) \, d\beta$  (App. \ref{app:symm_kl} \eqref{eq:symm_kl_squareA}). 

For $\beta_0 = 0$ and $\beta_1 = 1$, we can confirm from \cref{eq:elbo}-\cref{eq:eubo} that $\eta_1 - \eta_0 = \gls{EUBO} - \gls{ELBO} = D_{KL}[\qphi || p_{\theta} ] + D_{KL}[p_{\theta} || \qphi]$.

Before presenting our proposed approach for choosing $\beta$ in the next section, we note that App. \ref{app:recursive_scheduling} describes a `coarse-grained' linear binning schedule from \citet{grosse2013annealing}, which allocates intermediate distributions based on the identity \cref{eq:symm_kl_square} and is evaluated as a baseline in Sec. \ref{sec:experiments}.

\begin{figure}[t]
    \centering
    \includegraphics[width=.4\textwidth]{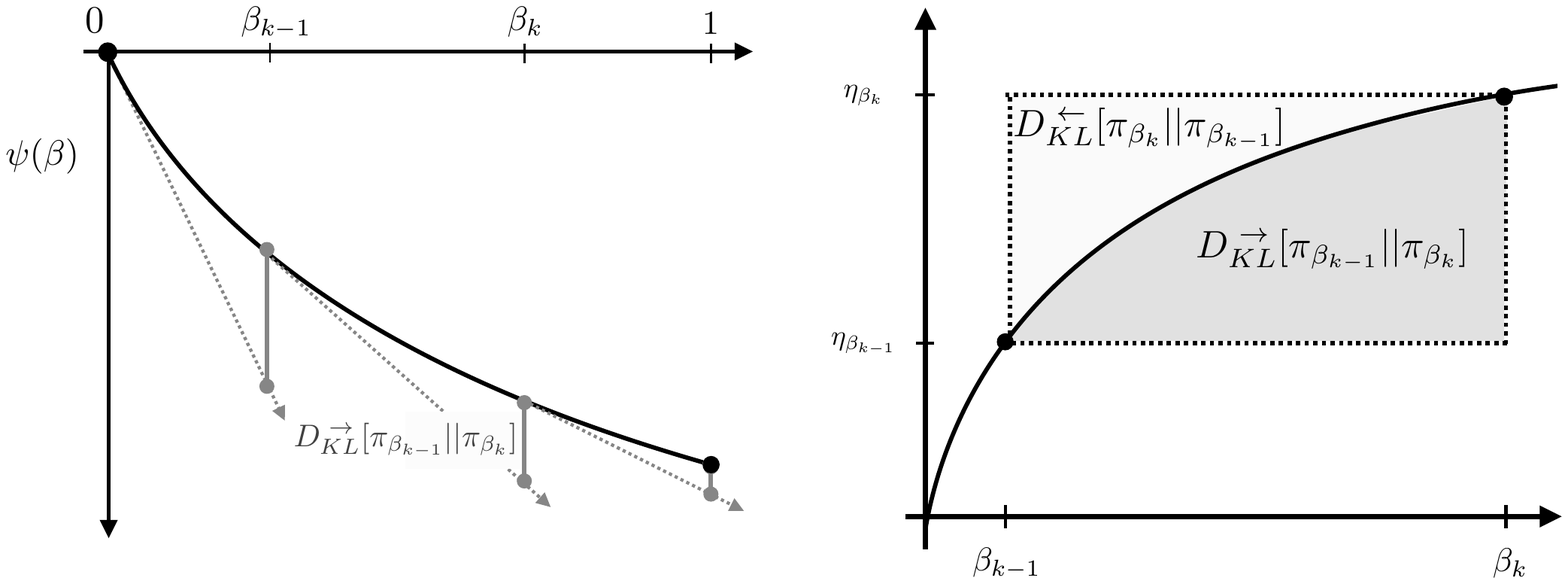} 
    \caption{Adding the KL divergences in each direction, we can visualize the symmetrized KL divergence as the area of a rectangle.  The curvature of the \gls{TVO} integrand suggests which direction of the KL divergence is larger, with the divergence becoming symmetric when $\eta_{\beta}$ is linear in $\beta$ (see App. \ref{app:grosse_perspective}). }\label{fig:symm_kl_one_piece}
\end{figure}

\begin{figure*}[t]
    \begin{minipage}{.475\textwidth}
     \vspace*{-.425cm}
    \includegraphics[width=.85\textwidth]{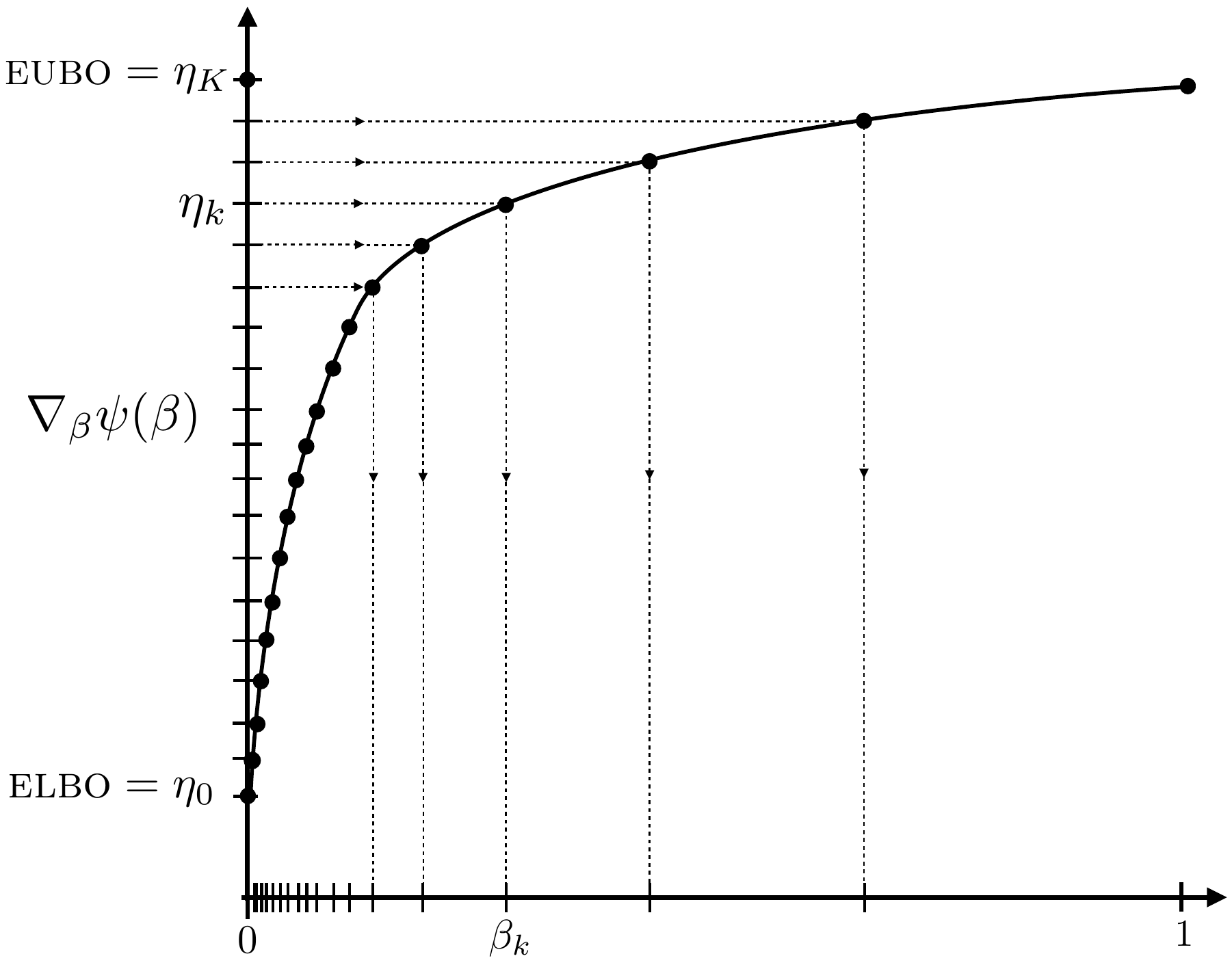}
    \vspace*{-.285cm}
    \captionof{figure}{By enforcing equal spacing in the mean parameter space, our moments schedule naturally `adapts' by allocating more partitions to regions where the integrand is changing quickly.} \label{fig:moment_schedule}
    \end{minipage}\hspace{.045\textwidth}%
    \begin{minipage}{.475\textwidth}
    \centering
    \vspace{.05cm}
    \includegraphics[width=.875\textwidth, trim=0 0 0 0, clip]{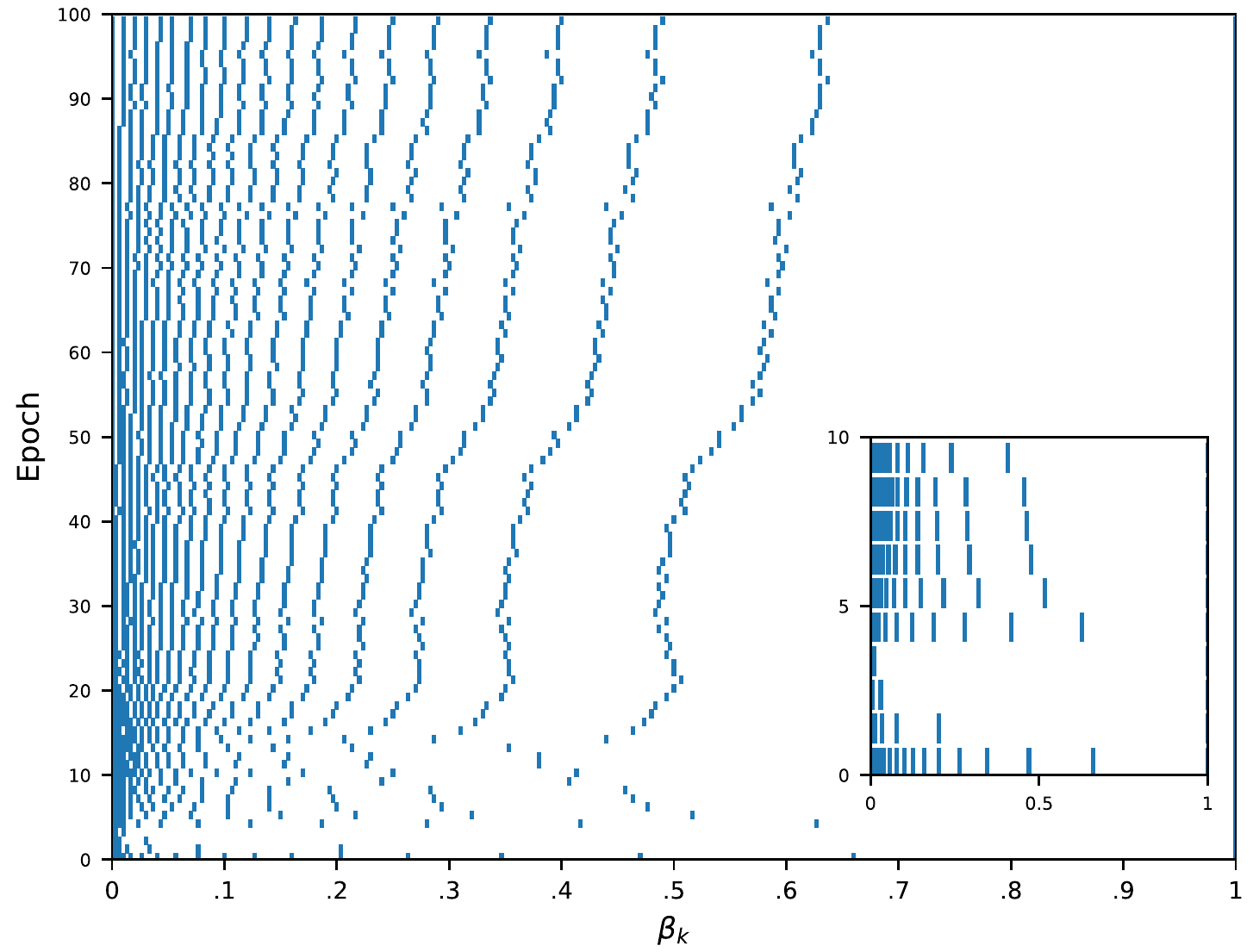}
    \vspace{-.31cm}
    \captionof{figure}{\centering We visualize placement of $\beta_k$ for our moments-spacing schedule across the first 100 epochs, with $K=20$.  Most $\beta_k$ concentrate near $0$ in early epochs, but spread out as training proceeds and the integrand becomes flatter as a function of $\beta$.} 
    \label{fig:moment_spacing}
    \end{minipage}
\end{figure*}


\section{Moment-Spacing Schedule} \label{sec:schedules}
\citet{masrani2019thermodynamic} observe that \gls{TVO} performance can depend heavily on the choice of partition schedule $\mathcal{P}_{\beta}$, and propose log-uniform spacing of $\{ \beta_2,...\beta_{K-1} \}$ with grid search over the initial $\beta_1$. 

Instead, we propose choosing $\beta_k$ to yield equal spacing in the $y$-axis of the \gls{TVO} integrand $\eta_{\beta} = \mathbb{E}_{\pibeta}[\log \frac{\pxz}{\qzx}]$, which corresponds to Lebesgue integration rather than Riemann integration in Fig. \ref{fig:moment_schedule}.
This scheduling arises naturally from our exponential family in Sec. \ref{sec:expfamily}, with the mean parameters $\eta_{\beta}$ corresponding to the dual parameters (\citet{wainwrightjordan} Sec. 3.5.2).
Equal spacing in the mean parameters also corresponds to the `moment-averaged' path of \citet{grosse2013annealing}, which was shown to yield robust estimators and natural generative samples from intermediate $\pibeta$ in the context of \gls{AIS}.

Given a budget of intermediate distributions $K = | \mathcal{P}_{\beta} |$, we seek $\beta_k$ such that $\eta_{\beta_k}$ are uniformly distributed between the endpoints $\eta_0 = \gls{ELBO}$ and $\eta_1 = \gls{EUBO}$ (see \cref{eq:elbo}-\cref{eq:eubo})
\begin{align}
\beta_k = \eta^{-1}_{\beta}\left( \frac{k}{K} \cdot \gls{ELBO} + (1-\frac{k}{K} ) \cdot \gls{EUBO} \right). \label{eq:moments_schedule}
\end{align}

We use $\eta^{-1}_{\beta}(\mu)$ to indicate the value of the natural parameter $\beta$ such that the expected sufficient statistics $\eta_{\beta}$ match a desired target $\mu$.  This mapping between parameterizations is known as the Legendre transform  and can be a difficult optimization in its own right \cite{wainwrightjordan}.

However, in the \gls{TVO} setting, estimating moments $\eta_{\beta}$ for a given $\beta$ simply involves reweighting and normalizing the importance samples using \gls{SNIS} in \eqref{eq:snis}.  Equipped with this cheap evaluation mechanism, we can apply binary search to find the $\beta_k$ with a given expectation value $\eta_{\beta_k}$, as in \eqref{eq:moments_schedule}.  We update our choice of schedule at the end of each epoch, and provide further implementation details in App. \ref{app:implementation}.

We visualize an example of our moments schedule in Fig. \ref{fig:moment_schedule}.  Note that uniform spacing in $\eta$ does not imply uniform spacing in $\beta$, since the Legendre transform is non-linear.  The resulting spacing in the $x$-axis reflects how quickly $\eta_{\beta}$ is changing as a function of $\beta$, matching the intuition that we should allocate more points in regions where the integrand is changing quickly.  Our moment-spacing schedule thus adapts to the shape of the \gls{TVO} integrand, which can change significantly across training (Fig. \ref{fig:moment_spacing}).   The integrand itself reflects the degree of posterior mismatch, since the curve will be flat when $\qzx = \pzx$, with $\eta_{\beta} = \log \px$ $ \forall \beta$. 
On the other hand, an integrand rising sharply away from $\beta=0$ indicates a poor proposal, with only several importance samples dominating the \gls{SNIS} weights.


\begin{figure*}[t]
    \centering 
    \vspace*{.25cm}
       \begin{subfigure}{.5\textwidth}
            \includegraphics[width=.95\textwidth]{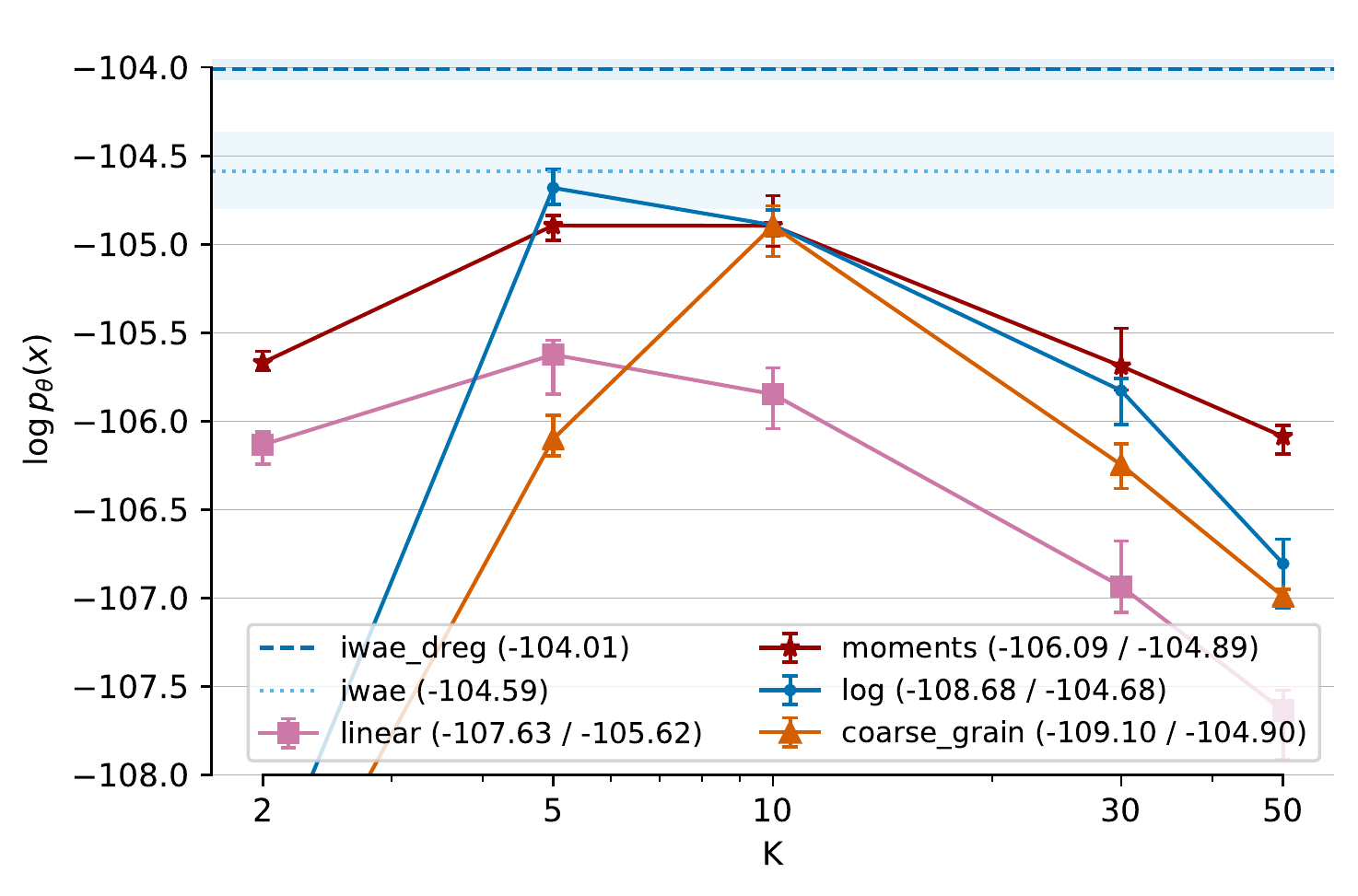}
            \caption{\gls{TVO} with \textsc{reinforce} Gradients}\label{subfig:omni_tvo_schedules}
    \end{subfigure}%
    \begin{subfigure}{.5\textwidth}
                \includegraphics[width=.95\textwidth]{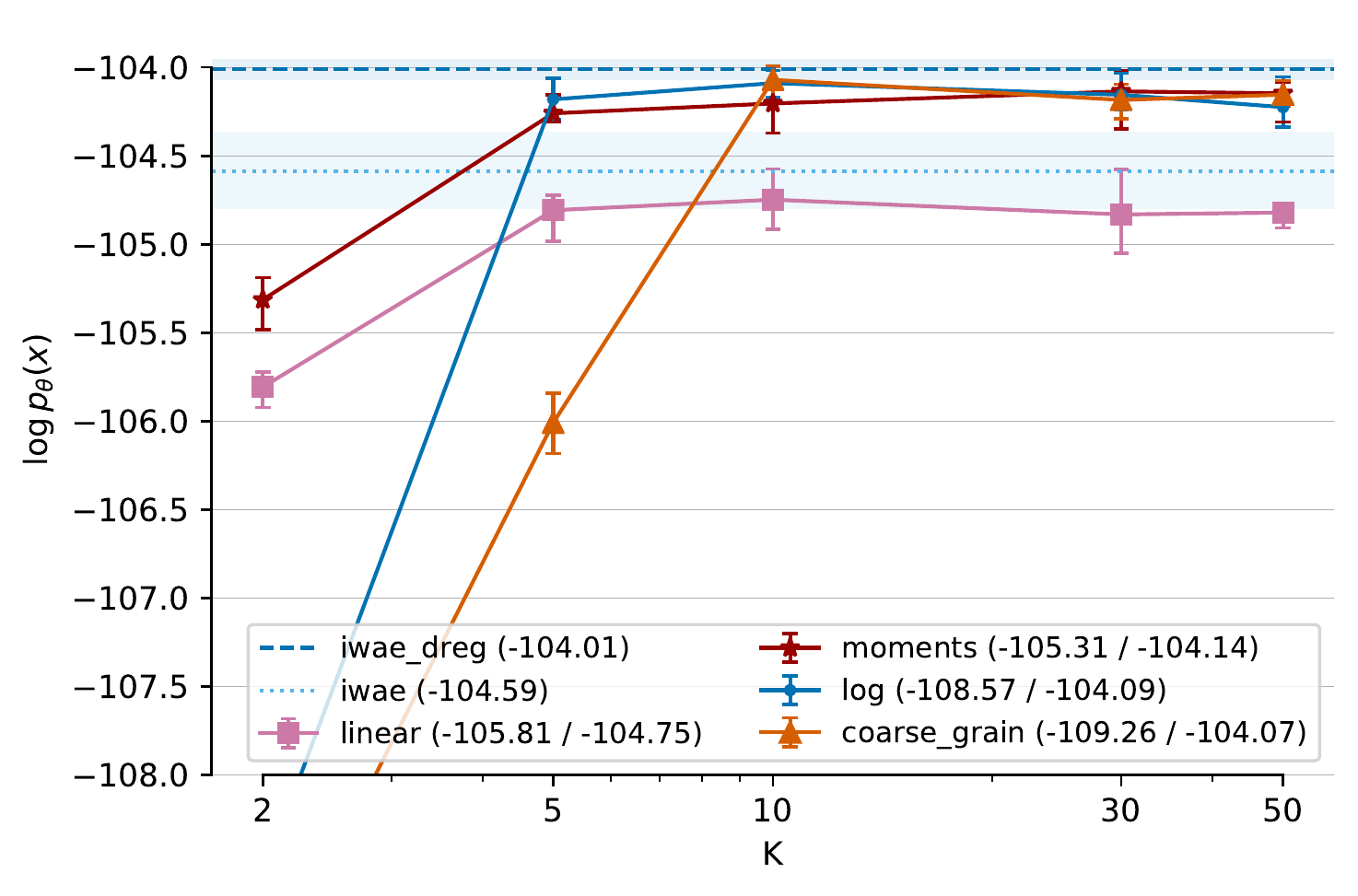}
                \caption{\gls{TVO} with Doubly-Reparameterized Gradients}\label{subfig:omni_reparam_schedules}
        \end{subfigure}
        \centering
        \caption{ Scheduling Performance by $K$ on Omniglot, with $S=50$.  Legend shows (min / max) test $\log p_{\theta}(\vx)$ across $K$.} \label{fig:omni_schedules}
\end{figure*}
\newcommand{\fphi}{ f(\vz) }

\section{Doubly-Reparameterized TVO Gradient}
\label{sec:optimization}
To optimize the \gls{TVO}, \citet{masrani2019thermodynamic} derive a $\REINFORCE$-style gradient estimator (see their App. F), which provides lower variance gradients and improved performance with discrete latent variables.   Writing $\lambda$ to denote $\{ \phi, \theta \}$, with $w = \pxz / \qzx$ and $\tpi_{\beta}(\vx, \vz)$ as in \eqref{eq:geomix}, we obtain gradients for expectations of arbitrary $\fphi$ under $\pi_{\beta}$, with the \gls{TVO} integrand corresponding to $\fphi = \log w$,
\begin{align}
    & \frac{d}{d \lambda} \, \mathbb{E}_{\pibeta}[ \fphi ] = \nonumber \\
    & \phantom{=} \mathbb{E}_{\pibeta} [\frac{d}{d \lambda} \fphi ] + \text{Cov}_{\pibeta}\left[ \fphi , \frac{d}{d \lambda}  \log \tpi_{\beta}(\vx, \vz) \right]\label{eq:tvo_grad} 
\end{align}
However, when $\vz_i \sim \qzx$ can be reparameterized via $\vz_i = z(\epsilon_i, \phi), \, \epsilon_i \sim p(\epsilon)$, we can improve the estimator in \cref{eq:tvo_grad} by more directly incorporating $\fphi$ gradient information.  To this end, we derive a doubly-reparameterized gradient estimator in App. \ref{app:tvo_integrand_gradients} 
\begin{align}
    \frac{d }{d \phi}  \mathbb{E}_{\pibeta} \left[ \fphi  \right] =~\mathbb{E}_{\pibeta} & \left[ \frac{d }{d \phi} \fphi  - \, \beta \cdot \frac{\partial \vz}{\partial \phi} \frac{\partial \fphi  }{\partial \vz} \right]  \nonumber \\[.75ex]
      \, \, + \, (1-\beta) \, \text{Cov}_{\pi_{\beta}}  &\left[ \fphi \, ,  \, \beta \cdot \frac{\partial \vz}{\partial \phi}  \frac{\partial \log w }{\partial \vz} \right].\label{eq:tvo_reparam_grad_main}
\end{align}
Doubly-reparameterized gradient estimators avoid a known signal-to-noise ratio issue for inference network gradients \cite{rainforth2018tighter}, using a second application of the reparameterization trick within the expanded total derivative~\cite{tucker2018doubly}.  We use a simplified form of \cref{eq:tvo_reparam_grad_main} (see App. \ref{app:tvo_integrand_gradients} \cref{eq:reparam_tvo}) for learning $\phi$ and \cref{eq:tvo_grad} for learning $\theta$.

Comparing the covariance terms of \cref{eq:tvo_grad} and \cref{eq:tvo_reparam_grad_main}, note that $\frac{d}{d \lambda}  \log \tpi_{\beta}(\vx, \vz)$ and $\beta \, \frac{\partial \vz}{\partial \phi}  \frac{\partial \log w }{\partial \vz}$  differ by their differentation operator and a factor of $\log q_{\phi}$ due to reparameterization, with  $\log \tpi_{\beta} = \log q_{\phi} + \beta \log w $. 

Further, the effect of the partial derivative $\frac{\partial \vz}{\partial \phi} \frac{\partial \fphi }{\partial \vz}$ in the first term of \cref{eq:tvo_reparam_grad_main} linearly decreases as  $\beta \rightarrow 1$ and $\pi_{\beta}(\vz|\vx)$ has less dependence on $\phi$.

Finally, we see that \cref{eq:tvo_reparam_grad_main} passes two basic sanity checks, with the covariance correction term vanishing at both endpoints.  At $\beta = 0$, we recover the gradient of the \gls{ELBO}, $\frac{d }{d \phi} \mathbb{E}_{\pi_0}[ \fphi ] = \mathbb{E}_{z(\epsilon, \phi)}[ \frac{d }{d \phi} \fphi ]$. At $\beta = 1$, note that the $\frac{\partial \vz}{\partial \phi} \frac{\partial \fphi  }{\partial \vz}$ term cancels when expanding $\frac{d }{d \phi} \fphi$, leaving $\frac{d }{d \phi} \mathbb{E}_{\pi_1} [ f_{\phi}(\vz) ] = \mathbb{E}_{p_{\theta}}[ \frac{\partial}{\partial \phi}f_{\phi}(\vz) ]$.  This is to be expected for expectations under $\pzx$, since the derivative with respect to $\phi$ passes inside the expectation and $\frac{\partial \vz}{\partial \phi} = 0$.



\section{Related Work}
\label{sec:related_work}


Thermodynamic integration (\gls{TI}) is a strategy for estimating partition function ratios or free energy differences in simulations of physical systems \cite{ogata1989monte, gelman1998simulating, frenkel2001understanding}, and also finds applications in model selection for phylogenetics \cite{lartillot2006computing,xie2011improving}. 

Physics applications of \gls{TI} often involve sampling forward and reverse state trajectories \cite{frenkel2001understanding,habeck2017model}, as might be done using \gls{MCMC} transition operators.  Indeed, upper and lower bounds identical to those in the \gls{TVO} are used to evaluate bidirectional Monte Carlo \cite{grosse2016measuring}.  A body of recent work `bridging the gap' between \gls{VI} and \gls{MCMC} \cite{Salimans2015, li2017approximate,caterini2018hamiltonian, huang2018improving, ruiz2019contrastive, lawson2019energy} might thus provide a basis for practical improvements in thermodynamic variational inference. 

Several recent \gls{VI} objectives also naturally appear within the \gls{TVO} framework.
As we show in App. ~\ref{app:renyi}, each log-partition function $\log Z_{\beta}(\vx)$ ~\cref{eq:partition_renyi} in our exponential family corresponds to a Renyi divergence \gls{VI} objective \cite{li2016renyi} with order $\alpha = 1-\beta$.  The $\textsc{cubo}$ objectives of \citet{dieng2017variational} correspond to upper bounds on $\log \px$ and log partition functions with $\beta \in [1,2]$.  From our exponential family perspective, there is no explicit restriction that our natural parameters $\beta$ remain in the unit interval, with the $\chi^2-$divergence at $\beta=2$ of notable interest \cite{cortes2010learning}.  
\citet{bamler2017perturbative,bamler2019tightening} also apply a Taylor series approach to obtain tighter bounds on $\log \px$, although the expansion is with respect to the importance weights $ T(\vx, \vz) = \logiw$ rather the natural parameter $\beta$.



\begin{figure*}[h]
    \vspace*{.25cm}
    \begin{minipage}{.5\textwidth}
    \subcaptionbox{Omniglot Test $\log \px$}{
    \includegraphics[width=.95\textwidth]{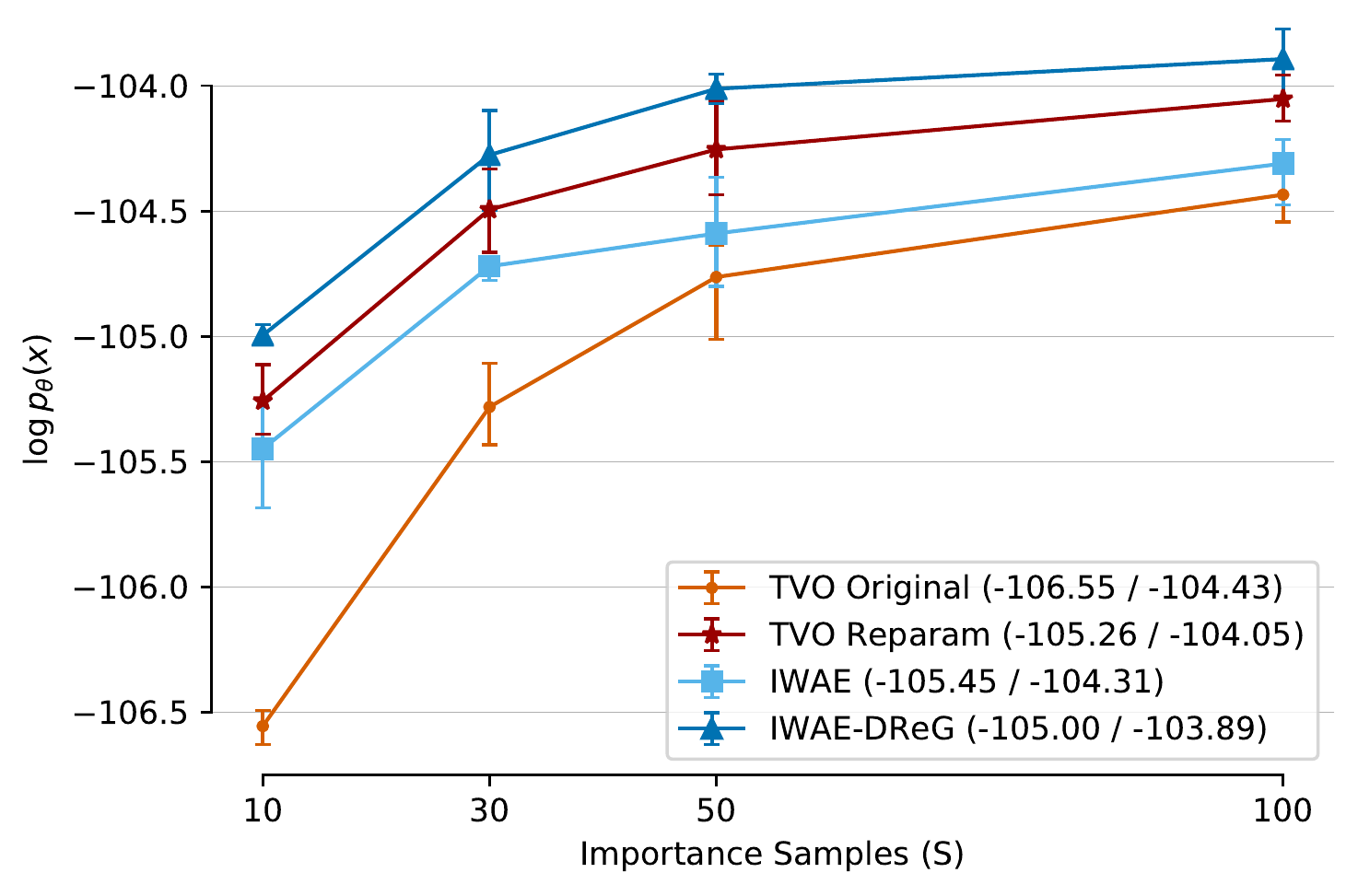}
    }\label{fig:iwaes_logpx}
    \end{minipage}%
    \begin{minipage}{.5\textwidth}
    \subcaptionbox{Omniglot Test $D_{KL}[\qzx||\pzx]$}{
    \includegraphics[width=.95\textwidth]{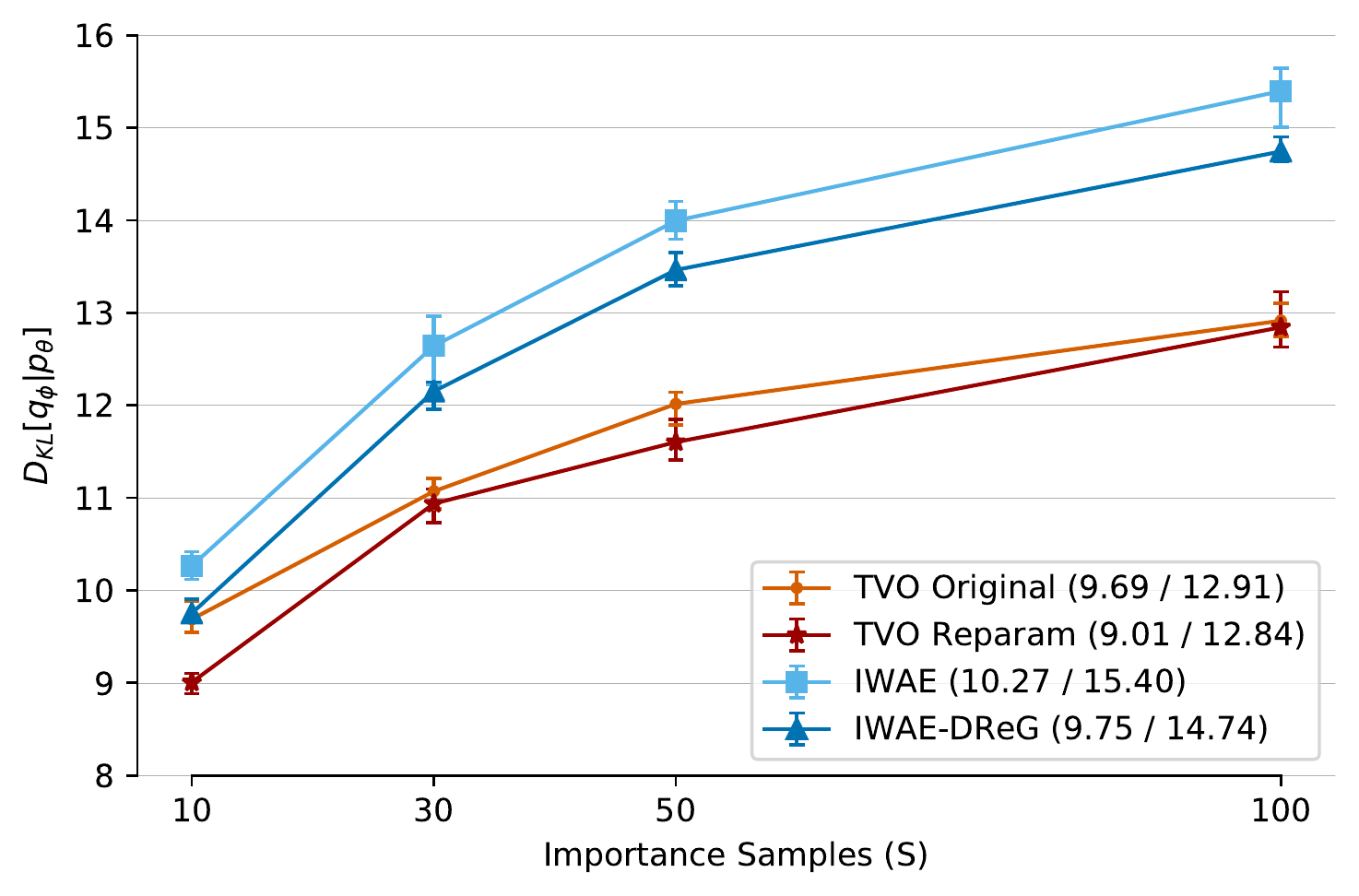}
    } \label{fig:iwaes_kl}
    \end{minipage}%
    \caption{\centering Model Learning and Inference by $S$ with $K=5$. Legend shows (min / max) values across $S$.} \label{fig:iwaes} 
\end{figure*}

\section{Experiments}\label{sec:experiments} 
We investigate the effect of our moment-spacing schedule and reparameterization gradients using a continuous latent variable model on the Omniglot dataset.  We estimate test $\log p_\theta(\vx)$ using the \textsc{iwae} bound \cite{Burda2015} with 5k samples, and use $S=50$ samples for training unless noted.   In all plots, we report averages over five random seeds, with error bars indicating min and max values.  We describe our model architecture and experiment design in App. \ref{app:experiments}, \footnote{\url{https://github.com/vmasrani/tvo_all_in}}  with runtimes and additional results on binary {\scshape mnist} in App. \ref{app:addl_results}. 
\vspace*{-.45cm}
\paragraph{Moment Spacing Dynamics}
 We seek understand the dynamics of our moment spacing schedule in Fig. \ref{fig:moment_spacing}, visualizing the choice of $\beta$ points across training epochs with $K=20$.  Our intermediate distributions concentrate near $\beta=0$ at the beginning of training, since $\qzx$ and $\pzx$ are mismatched and the \gls{TVO} integrand rises sharply away from $\qzx$.  This effect is particularly dramatic within the first five epochs.

While the curve is still fairly noisy within the first twenty epochs, it begins flatten as training progresses and $\qzx$ learns to match $\pzx$.  This is reflected in the $\beta_k$ achieving a given proportion of the moments difference (\gls{EUBO}- \gls{ELBO}) moving to higher values.
We found the moment-scheduling partitions to be relatively stable after 100 epochs.



\paragraph{Grid Search Comparison}
Next, we fix $K=2$ with only $\beta_1$ chosen by the moment spacing schedule.  We compare against grid search in Fig. ~\ref{fig:single_beta_omni} and Fig. \ref{fig:single_beta_mnist} (App. \ref{app:addl_results}), and plot test $\log p_\theta(\vx)$ as a function of $\beta_1 \in [0, 1]$ across 25 static values.  We report the value of $\beta_1$ for our moments schedule at the final epoch, which indicates where $\eta_{\beta_1}$ is halfway between our estimated $\gls{ELBO}$ and $\gls{EUBO}$.

We find that our adaptive scheduling matches the best performance from grid search, with the optimal intermediate distribution occurring at $\beta_1 \approx 0.3$ on both datasets.  With a single, properly chosen intermediate distribution, we find that the \gls{TVO} can achieve notable improvements over the \gls{ELBO} at minimal additional cost.

\paragraph{Evaluating Scheduling Strategies}
From a numerical integration perspective, the \gls{TVO} bounds should become arbitrarily tight as $K \rightarrow \infty$.  However, \citet{masrani2019thermodynamic} observe that additional partitions can be detrimental for learning in practice.
We thus investigate the performance of our moment spacing schedule with a varying number of partitions.
We plot test log likelihood at $K=\{2,5,10,30,50\}$, and compare against three scheduling baselines: linear, log-uniform spacing, and the `coarse-grained' schedule from \citet{grosse2013annealing} (see App. \ref{app:recursive_scheduling}).  We begin the log-uniform spacing at $\beta_1 = 0.025$, a choice which results from grid search over $\beta_1$ for $K>2$ in \citet{masrani2019thermodynamic}.

We observe in Fig. \ref{subfig:omni_tvo_schedules} that the moment scheduler provides the best performance at high and low $K$, while the log-uniform schedule can perform best for particular $K$.  As previously observed, all scheduling mechanisms still suffer degradation in performance at large $K$. 

\paragraph{Reparameterized TVO Gradients}
While our scheduling techniques do not address the detrimental effect of using many intermediate $\beta$, we now investigate the use of our reparameterization gradient estimator from Sec.~\ref{sec:optimization}.   Repeating the previous experiment in Fig. \ref{subfig:omni_reparam_schedules},  we find that reparameterization helps preserve competitive performance for high $K$ and improves overall model likelihoods.  Our moments schedule is still particularly useful at low $K$, while the various scheduling methods converge to similar performance with many partition points.  All scheduling techniques will be equivalent in the limit, as discussed in App. \ref{app:grosse_perspective}.

\paragraph{Comparison with IWAE} 
Finally, we compare \gls{TVO} with moments scheduling against the \gls{IWAE} \cite{Burda2015} and doubly reparameterized \gls{IWAE} {\scshape dreg} \cite{tucker2018doubly} for model learning and posterior inference.  It is interesting to note that \gls{IWAE} corresponds to a direct estimate of $\psi(1)$, with the \gls{SNIS} normalizer $\sum_{i=1}^{S} w_i^1$ in \gls{TVO} \eqref{eq:snis} appearing inside the $\log$.

In Fig. \ref{fig:iwaes}, we observe that \gls{TVO} with reparameterization gradients achieves model learning performance in between that of \gls{IWAE} and \gls{IWAE} {\scshape dreg}, with lower KL divergences across all values of $S$.  We repeat this experiment for {\scshape mnist} in App. \ref{app:addl_results} Fig. \ref{fig:iwaes_mnist}, where \gls{TVO} matches \gls{IWAE} {\scshape dreg} model learning with better inference.  Although we tend to obtain lower $D_{KL}$ with lower model likelihood, we do not observe strong evidence of the signal-to-noise ratio issues of \citet{rainforth2018tighter} on either dataset.  \gls{TVO} with reparameterization thus appears to provide a favorable tradeoff between model learning and posterior inference. %

\section{Conclusion}
\label{sec:conclusions} 
In this work, we interpret the geometric mixture curve found in thermodynamic integration (\gls{TI}), annealed importance sampling (\gls{AIS}), and the Thermodynamic Variational Objective (\gls{TVO}), using the Bregman duality of exponential families.  We leveraged this approach to characterize the gap in \gls{TVO} lower and upper bounds as a sum of KL divergences along a given path, and presented an adaptive scheduling technique based on the mean parameterization of our exponential family.  Finally, we derived a doubly-reparameterized gradient estimator for terms in the \gls{TVO} integrand.


The use of self-normalized importance sampling (\gls{SNIS}) to estimate expectations under $\pi_{\beta}$ may still be a key limitation of the \gls{TVO} (see \citet{masrani2019thermodynamic}), although we relied on the efficiency of \gls{SNIS} for our moment-spacing schedule.  Improved \gls{MCMC} estimators that can be integrated with end-to-end learning of $q_{\phi}(\vz|\vx)$ and $p_{\theta}(\vx,\vz)$ remain an intruiging direction for future work.   In this study, we did not observe performance gains using equal spacing in either the KL or symmetrized KL divergence, but alternative schedules might also be motivated via physical interpretations \cite{andresen1994constant,salamon2002simple,sivak2012thermodynamic}.   We thus hope that our work can encourage further  contributions in thermodynamic variational inference (\textsc{tvi}), a class of methods combining insights from \gls{VI}, \gls{MCMC}, and statistical physics.

\section*{Acknowledgements}
The authors would like to thank Tuan Anh Le for clarifying the interpretation of the symmetrized KL divergence, and an anonymous reviewer for suggesting the connection with Lebesgue integration.  RB thanks Kyle Reing, Artemy Kolchinsky, and Frank Nielsen for helpful discussions. VM thanks David Dehaene for his suggestion of reparameterized TVO gradients.  This paper builds closely upon a workshop paper by RB, GV, and AG.

RB and GV acknowledge support from the Defense Advanced Research Projects Agency (DARPA) under awards FA8750-17-C-0106 and W911NF-16-1-0575.  VM acknowledges the support of the Natural Sciences and Engineering Research Council of Canada (NSERC) under award number PGSD3-535575-2019 and the British Columbia Graduate Scholarship, award number 6768. VM/FW acknowledge the support of the Natural Sciences and Engineering Research Council of Canada (NSERC), the Canada CIFAR AI Chairs Program, and the Intel Parallel Computing Centers program. This material is based upon work supported by the United States Air Force Research Laboratory (AFRL) under the Defense Advanced Research Projects Agency (DARPA) Data Driven Discovery Models (D3M) program (Contract No. FA8750-19-2-0222) and Learning with Less Labels (LwLL) program (Contract No.FA8750-19-C-0515). Additional support was provided by UBC's Composites Research Network (CRN), Data Science Institute (DSI) and Support for Teams to Advance Interdisciplinary Research (STAIR) Grants. This research was enabled in part by technical support and computational resources provided by WestGrid (https://www.westgrid.ca/) and Compute Canada (www.computecanada.ca).


\bibliography{main}
\bibliographystyle{icml2020}

\clearpage
\appendix
\twocolumn
\section{Conjugate Duality} \label{app:duality}
The Bregman divergence associated with a convex function $f: \Omega \rightarrow \mathbb{R}$ can be written as \cite{Banerjee2005}:
\begin{align*}
    D_{B_f}[p:q] = f(p) + f(q) - \langle p-q, \nabla f(q) \rangle
\end{align*}
The family of Bregman divergences includes many familiar quantities, including the KL divergence corresponding to the negative entropy generator $f(p)= -\int p \log p \, d\omega$.  Geometrically, the divergence can be viewed as the difference between $f(p)$ and its linear approximation around $q$.  Since $f$ is convex, we know that a first order estimator will lie below the function, yielding $D_{f}[p:q] \geq 0$.

For our purposes, we can let $f \triangleq \psi(\beta) = \log Z_{\beta}$ over the domain of probability distributions indexed by natural parameters of an exponential family (e.g. \eqref{eq:exp_fam}) :
\begin{align}
    D_{\psi}[\beta_p : \beta_q] = \psi(\beta_p) - \psi(\beta_q) - \langle \beta_p - \beta_q, \nabla_{\beta} \psi(\beta_q) \rangle \label{eq:psi_def}
\end{align}
This is a common setting in the field of information geometry \cite{amari2016information}, which introduces dually flat manifold structures based on the natural parameters and the mean parameters.

\subsection{KL Divergence as a Bregman Divergence}\label{app:breg_kl}
For an exponential family with partition function $\psi(\beta)$ and sufficient statistics $T(\omega)$ over a random variable $\omega$, the Bregman divergence $D_{\psi}$ corresponds to a KL divergence.  Recalling that $\nabla_{\beta} \psi(\beta)= \eta_{\beta} = \mathbb{E}_{\pibeta}[T(\omega)]$ from \cref{eq:deriv}, we simplify the definition \cref{eq:psi_def} to obtain
\begin{align}
D_{\psi}[\beta_p : \beta_q] &= \psi(\beta_p) - \psi(\beta_q) -  \beta_p \cdot \eta_{q} + \beta_q \cdot \eta_{q} \nonumber\\
&= \psi(\beta_p) - \psi(\beta_q) -  \mathbb{E}_{q} [\beta_p \cdot T(\omega)] \nonumber \\
&\phantom{= \psi(\beta_p) - \psi(\beta_q) } + \mathbb{E}_{q} [\beta_q \cdot T(\omega)] \nonumber \\[1.5ex]
&= \underbrace{\mathbb{E}_{q} \big[\beta_q \cdot T(\omega) - \psi(\beta_q)\big] + \mathbb{E}_{q}[\pi_0(\omega)]}_{ \log q(\omega) } \nonumber \\
&\phantom{=\mathbb{E}} - \underbrace{\mathbb{E}_{q} \big[\beta_p \cdot T(\omega) - \psi(\beta_p)\big] - \mathbb{E}_{q}[\pi_0(\omega)]}_{ \log p(\omega) } \nonumber \\
&= \mathbb{E}_{q} \log \frac{q(\omega)} {p(\omega) } \nonumber \\
&= D_{KL}[q(\omega) || p(\omega)] \label{eq:breg_kl}
\end{align}
where we have added and subtracted terms involving the base measure $\pi_0(\omega)$, and used the definition of our exponential family from \eqref{eq:exp_fam}.  The Bregman divergence $D_{\psi}$ is thus equal to the KL divergence with arguments reversed.

\subsection{Dual Divergence}\label{app:dual_divergence}
We can leverage convex duality to derive an alternative divergence based on the conjugate function $\psi^{*}$. 
\begin{align}
\nonumber \\
\psi^{*}(\eta) &= \sup \limits_{\beta} \, \eta \cdot \beta - \psi(\beta) \quad \implies \quad \eta = \nabla_{\beta} \,  \psi(\beta) \nonumber \\
&= \eta \cdot \beta_{\eta} - \psi(\beta_{\eta}) \label{eq:conj}
\end{align}
The conjugate measures the maximum distance between the line $\eta \cdot \beta$ and the function $\psi(\beta)$, which occurs at the unique point $\beta_{\eta}$ where $\eta = \nabla_{\beta} \psi(\beta)$.  This yields a bijective mapping between $\eta$ and $\beta$ for minimal exponential families \cite{wainwrightjordan}.  Thus, a distribution $p$ may be indexed by either its natural parameters $\beta_p$ or mean parameters $\eta_p$.

Noting that $(\psi^{*})^{*} = \psi(\beta) = \sup_{\eta} \eta \cdot \beta - \psi^{*}(\eta)$ \cite{boyd2004convex}, we can use a similar argument as above to write this correspondence as $\beta = \nabla_{\eta}  \psi^{*}(\eta)$.  We can then write the dual divergence $D_{\psi^{*}}$ as:
\begin{align}
    D_{\psi^{*}}[\eta_p : \eta_q] &= \psi^{*}(\eta_p) - \psi^{*}(\eta_q) - \langle \eta_p - \eta_q, \nabla_{\eta} \, \psi^{*}(\eta_q) \rangle \nonumber \\
    &= \psi^{*}(\eta_p) - \underline{\psi^{*}(\eta_q)} - \eta_p \cdot \beta_q  + \underline{\eta_q  \cdot \beta_q} \nonumber \\
    &=  \psi^{*}(\eta_p) + \underline{\psi(\beta_q)} - \eta_p  \cdot  \beta_q   \label{eq:dual_div}
\end{align}
where we have used ~\cref{eq:conj} to simplify the underlined terms.  Similarly,
\begin{align}
  D_{\psi}[\beta_p : \beta_q] &= \psi(\beta_p) - \psi(\beta_q) - \langle \beta_p - \beta_q, \nabla_{\beta} \psi(\beta_q) \rangle \nonumber \\
   &= \psi(\beta_p) - \underline{\psi(\beta_q)} - \beta_p  \cdot  \eta_q  +  \underline{\beta_q  \cdot  \eta_q} \nonumber \\
  &= \psi(\beta_p) + \underline{\psi^{*}(\eta_q)} -  \beta_p \cdot \eta_q \label{eq:primal_div}
\end{align}
Comparing \eqref{eq:dual_div} and \eqref{eq:primal_div}, we see that the divergences are equivalent with the arguments reversed, so that:
\begin{align}
\vspace{-.25cm}
D_{\psi}[\beta_p : \beta_q] = D_{\psi^{*}}[\eta_q : \eta_p] \label{eq:rev_args}
\vspace{-.25cm}
\end{align}

This indicates that the Bregman divergence $D_{\psi^{*}}$ should also be a KL divergence, but with the same order of arguments.  We derive this fact directly in \eqref{eq:kl_as_breg} , after investigating the form of the conjugate function $\psi^{*}$.

\subsection{Conjugate $\psi^{*}$ as Negative Entropy}\label{app:entropy_kl}
We first treat the case of an exponential family with no base measure $\pi_0(\omega)$, with derivations including a base measure in App. \ref{app:conj_kl}.  For a distribution $p$ in an exponential family, indexed by $\beta_p$ or $\eta_p$, we can write $\log p (\omega) = \beta_p \cdot T(\omega) - \psi(\beta)$.  Then, \eqref{eq:conj} becomes:
\begin{align}
    \vspace{-.25cm}
    \psi^{*}(\eta_p) &= \beta_p \cdot \eta_p  - \psi(\beta_p) \label{eq:dualderiv_0} \\
    &= \beta_p \cdot \mathbb{E}_p [T(\omega)] - \psi(\beta_p) \label{eq:dualderiv_eta} \\
    &= \mathbb{E}_{p} \log p(\omega) \label{eq:dualderiv_prob} \\
    &= -H_{p}(\omega) \label{eq:dualderiv_1}
    \vspace{-.25cm}
\end{align}
since $\beta_p$ and $\psi(\beta_p)$ are constant with respect to $\omega$.   Utilizing $\psi^{*}(\eta_p) = \mathbb{E}_{p} \log p(\omega)$ from above, the dual divergence with $q$ becomes:
\begin{align}
\vspace*{-.25cm}
 D_{\psi^{*}}[\eta_p : \eta_q] &= \psi^{*}(\eta_p) - \psi^{*}(\eta_q) - \langle \eta_p - \eta_q, \nabla_{\eta} \, \psi^{*}(\eta_q) \rangle \nonumber \\
 &= \mathbb{E}_p \log p(\omega) - \underline{\psi^{*}(\eta_q)} - \eta_p \cdot \beta_q + \underline{\eta_q \cdot \beta_q} \nonumber  \\
 &= \mathbb{E}_p \log p(\omega) - \eta_p \cdot \beta_q + \underline{\psi(\beta_q)} \nonumber \\
 &= \mathbb{E}_p \log p(\omega) - \mathbb{E}_p [T(\omega) \cdot \beta_q] + \psi(\beta_q) \nonumber  \\
 &= \mathbb{E}_p \log p(\omega) - \mathbb{E}_p \log q(\omega) \nonumber  \\
 &= D_{KL}[p(\omega)||q(\omega)] \label{eq:kl_as_breg}
 \vspace*{-.25cm}
\end{align}

Thus, the conjugate function is the negative entropy and induces the KL divergence as its Bregman divergence \cite{wainwrightjordan}.  

Note that, by ignoring the base distribution over $\omega$, we have instead assumed that $\pi_0(\omega):= u(\omega)$ is uniform over the domain.   In the next section, we illustrate that the effect of adding a base distribution is to turn the conjugate function into a KL divergence, with the base $\pi_0(\omega)$ in the second argument.  This is consistent with our derivation of negative entropy, since $ D_{KL}[p_{\beta}(\omega)||u(\omega)] = -H_{p_{\beta}}(\Omega) + const$.

\subsection{Conjugate $\psi^{*}$ as a KL Divergence}\label{app:conj_kl}
  As noted above, the derivation of the conjugate $\psi^{*}(\eta)$ in \eqref{eq:dualderiv_0}-\eqref{eq:dualderiv_1} ignored the possibilty of a base distribution in our exponential family.  We see that $\psi^{*}(\eta)$ takes the form of a KL divergence when considering a base measure $\pi_0(\omega)$. 
\begin{align}
    \psi^{*}(\eta) &= \sup \limits_{\beta} \beta \cdot \eta - \psi(\beta) \label{eq:dual_opt} \\
    &= \beta_{\eta} \cdot \eta - \psi(\beta_{\eta}) \nonumber \\
    &=  \mathbb{E}_{\pi_{\beta_{\eta}}} [\beta_{\eta} \cdot T(\omega)] - \psi(\beta_{\eta}) \nonumber \\
    &= \mathbb{E}_{\pi_{\beta_{\eta}}}  [\beta_{\eta} \cdot T(\omega)]  - \psi(\beta_{\eta}) \pm \mathbb{E}_{\pi_{\beta_{\eta}}} [\log \pi_0(\omega) ] \nonumber \\
    &= \mathbb{E}_{\pi_{\beta_{\eta}}} [ \log \pi_{\beta_{\eta}(\omega)} - \log \pi_0(\omega)] \nonumber \\
    &= D_{KL}[\pi_{\beta_{\eta}}(\omega) || \, \pi_0(\omega)] \label{eq:conjugate_kl}
\end{align}
 Note that we have added and subtracted a factor of $\mathbb{E}_{\pi_{\beta_{\eta}}} \log \pi_0(\omega)$ in the fourth line, where our base measure $\pi_0(\omega)= q(\vz|\vx) $ in the case of the \gls{TVO}.  Comparing with the derivations in \eqref{eq:dualderiv_eta}-\eqref{eq:dualderiv_prob}, we need to include a term of $\mathbb{E}_p \pi_0(\omega)$ in moving to an expected log-probability $\mathbb{E}_p \log p(\omega)$, with the extra, subtracted base measure term transforming the negative entropy into a KL divergence.

In the \gls{TVO} setting, this corresponds to
\begin{align}
\psi^{*}(\eta) = D_{KL}[\pi_{\beta_{\eta}}(\vz|\vx) || \, q(\vz|\vx)] \, .
\end{align}

When including a base distribution, the induced Bregman divergence is still the KL divergence since, as in the derivation of \cref{eq:breg_kl}, both $\mathbb{E}_{p} \log p(\omega)$  and $\mathbb{E}_{p} \log q(\omega)$  will contain terms involving the base distribution $\mathbb{E}_{p} \log \pi_0(\omega)$.

\section{Renyi Divergence Variational Inference}\label{app:renyi}
In this section, we show that each intermediate partition function $\log Z_{\beta}$ corresponds to a scaled version of the Rényi VI objective $\mathcal{L}_{\alpha}$ \cite{li2016renyi}.

To begin, we recall the definition of Renyi's $\alpha$ divergence.
\begin{align*}
    D_{\alpha}[p||q] = \frac{1}{\alpha-1} \log \int q(\omega)^{1-\alpha} p(\omega)^{\alpha} d\omega
\end{align*}
Note that this involves geometric mixtures similar to \eqref{eq:partition_renyi}.   Pulling out the factor of $\log p(\vx)$ to consider normalized distributions over $\vz|\vx$, we obtain the objective of \citet{li2016renyi}.  This is similar to the \gls{ELBO}, but instead subtracts a Renyi divergence of order $\alpha$.
\begin{align}
 \psi(\beta) &=  \log \int q(\vz|\vx)^{1-\beta} p(\vx,\vz)^{\beta} d\vz \nonumber \\
 &= \beta \log p(\vx) - (1 -\beta) D_{\beta}[\pzx || \qzx]  \nonumber\\
 &= \beta \log p(\vx) - \beta \, D_{1-\beta}[\qzx || \pzx] \nonumber \\
 &:= \beta \, \mathcal{L}_{1-\beta} \nonumber 
 \end{align}
where we have used the skew symmetry property $ D_{\alpha}[p||q] = \frac{\alpha}{1-\alpha} D_{1-\alpha}[q||p]$ for $0 < \alpha < 1$ \cite{van2014renyi}.  Note that $\mathcal{L}_0 = 0$ and $\mathcal{L}_{1} = \log \px$ as in \citet{li2016renyi} and Sec. \ref{sec:expfamily}.

\section{TVO using Taylor Series Remainders}\label{app:taylor}

Recall that in Sec. \ref{sec:bregman_tvo}, we have viewed the KL divergence $D_{\psi} [\beta : \beta^{\prime}]$ as the remainder in a first order Taylor approximation of $\psi(\beta)$ around $\beta^{\prime}$.  The \gls{TVO} objectives correspond to the linear term in this approximation, with the gap in $\tvolb$ and $\tvoub$ bounds amounting to a sum of KL divergences or Taylor remainders.  Thus, the \gls{TVO} may be viewed as a first order method.

Yet we may also ask, what happens when considering other approximation orders?  We proceed to show that thermodynamic integration arises from a zero-order approximation, while the symmetrized KL divergence corresponds to a similar application of the fundamental theorem of calculus in the mean parameter space $\eta_{\beta} = \nabla_{\beta} \psi(\beta)$.  In App. \ref{app:higher_order_tvo}, we briefly describe how `higher-order' \gls{TVO} objectives might be constructed, although these will no longer be guaranteed to provide upper or lower bounds on likelihood.

We will repeatedly utilize the integral form of the Taylor remainder theorem, which characterizes the error in a $k$-th order approximation of $\psi(x)$ around $a$, with $\beta \in [a, x]$\footnote{We use generic variable $x$, not to be confused with data $\vx$, for notational simplicity.
}.  This identity can be derived using the fundamental theorem of calculus and repeated integration by parts (see, e.g. \cite{kountourogiannis2003derivation} and references therein):
\begin{align}
    R_k(x) = \int_{a}^{x} \frac{\nabla_{\beta}^{(k+1)} \psi(\beta)}{k!} (x - \beta)^k d\beta \label{eq:remainder1}
\end{align}

\subsection{Thermodynamic Integration as $0^{th}$ Order Remainder}\label{app:taylor_ti}

Consider a zero-order Taylor approximation of $\psi(1)$ around $a=0$, which simply uses $\psi(0)$ as an estimator.  Applying the remainder theorem, we obtain the identity ~\cref{eq:tvi} underlying thermodynamic integration in the \gls{TVO}:
\begin{align}
    \psi(1) &= \psi(0) + R_0(1) \\
     \psi(1) - \psi(0) &= \int_{0}^{1} \nabla_\beta \psi(\beta) d\beta \label{eq:fundamental} \\
    \log\px &= \int_{0}^{1} \E_{\pibeta} \bigg[ \logiw \bigg] d\beta \label{eq:ti_remainder}
\end{align}
where the last line follows as the definition of $\eta = \nabla_{\beta} \psi(\beta) = \nabla_{\beta} \log Z_{\beta}$ in ~\cref{eq:deriv}.

Note that this integration is symmetric, in that approximating $\psi(0)$ using $\psi(1)$ leads to an equivalent expression after reversing the order of integration.

\subsection{KL Divergence as $1^{st}$ Order Remainder}\label{app:taylor_kl}
We can apply a similar approach to the first order Taylor approximations to reinterpret the TVO bound gaps in \cref{eq:tvo_l} and \cref{eq:tvo_u}, although our remainder expressions will no longer be symmetric.  We will thus distinguish between estimating $\psi(x)$ around $a < x$ and $a > x$ using $R_1^{\rightarrow}(x)$ and  $R_1^{\leftarrow}(x)$, respectively, with the arrow indicating the direction of integration.

Estimating $\psi(\beta_k)$ using a first order approximation around $a = \beta_{k-1}$ as in the \gls{TVO} lower bound, the remainder exactly matches the definition of the Bregman divergence in \cref{eq:breg_def}:
\begin{align}
   R_1^{\rightarrow}(\beta_{k}) &= \psi(\beta_k)-\big( \underbrace{\psi(\beta_{k-1})+(\beta_k-\beta_{k-1})\nabla_{\beta}\psi(\beta_{k-1})}_{\text{First-Order Taylor Approx}} \big) \nonumber \\
   &= \int \limits_{\beta_{k-1}}^{\beta_k} \frac{\nabla_{\beta}^2 \psi(\beta)}{1!} (\beta_k - \beta)^1 d\beta \label{eq:remainder_example}
\end{align}
where \eqref{eq:remainder_example} corresponds to the Taylor remainder from \eqref{eq:remainder1}.
Recall that this Bregman divergence $D_{\psi}[\beta_{k}:\beta_{k-1}]$ corresponds to a KL divergence $D_{KL}^{\,\rightarrow}[\pi_{\beta_{k-1}} || \pi_{\beta_k}]$ and contributes to the gap in $\tvolb$.

Simplifying the Taylor remainder expression, with $\nabla^2_{\beta} \psi(\beta) =  \Var_{\pibeta} \logiw $, we obtain an integral representation of the KL divergence:
\begin{align}
    D_{KL}^{\,\rightarrow}[\pi_{\beta_{k-1}} || \pi_{\beta_k}] = \int \limits_{\beta_{k-1}}^{\beta_k}(\beta_k - \beta) \Var_{\pibeta} \logiw d\beta \label{eq:kl_var_lr}
\end{align}

Following similar arguments in the reverse direction, we can obtain an integral form for the \gls{TVO} upper bound gap $R_1^{\leftarrow}(\beta_{k-1}) = D_{KL}[\pi_{\beta_{k}} || \pi_{\beta_{k-1}}]$ via the first-order approximation of $\psi(\beta_{k-1})$ around $a = \beta_{k}$. 
\begin{align}
   R_1^{\leftarrow}(\beta_{k-1}) &= \psi(\beta_{k-1})-\big( \psi(\beta_{k})+(\beta_{k-1}-\beta_{k})\nabla_{\beta}\psi(\beta_{k}) \big) \nonumber \\[1.5ex]
   &= (\beta_{k}-\beta_{k-1})\nabla_{\beta}\psi(\beta_{k}) - ( \psi(\beta_{k}) - \psi(\beta_{k-1}) ) \nonumber \\[1.5ex]
   &= \int \limits_{\beta_{k}}^{\beta_{k-1}} \frac{\nabla_{\beta}^2 \psi(\beta)}{1!} (\beta_{k-1} - \beta)^1  d\beta \label{eq:left_remainder}
\end{align}
Note that the \gls{TVO} upper bound \cref{eq:tvo_u} arises from the second line, with $R_1^{\leftarrow}(\beta_{k-1}) \geq 0$ and $(\beta_{k}-\beta_{k-1}) \nabla_{\beta}\psi(\beta_{k})$ corresponding to a right-Riemann approximation.

Switching the order of integration in \eqref{eq:left_remainder}, we can write the KL divergence as
\begin{align}
&D_{KL}^{\,\leftarrow}[\pi_{\beta_{k}} || \pi_{\beta_{k-1}}] = \int \limits_{\beta_{k-1}}^{\beta_k}(\beta - \beta_{k-1}) \Var_{\pibeta} \logiw d\beta \label{eq:kl_var_rl}
\end{align}

While these integral expressions for the KL divergence may not be immediately intuitive, our use of the Taylor remainder theorem unifies their derivation with that of thermodynamic integration.  Alternative derivations may also be found in \citet{dabak2002relations}.

\subsection{Symmetrized KL Divergence}\label{app:symm_kl}
Combining the expressions for the KL divergence in Eq. \cref{eq:kl_var_lr} and \cref{eq:kl_var_rl} immediately leads to a known result relating the symmetrized KL divergence to the integral of the Fisher information along the geometric path  \cite{amari2016information,dabak2002relations}.
\begin{align}
\dsymm &= (\beta_k - \beta_{k-1}) \int \limits_{\beta_{k-1}}^{\beta_{k}} \Var_{\pibeta} \bigg[ \logiw d\beta \bigg] \label{eq:symm_kl_integralA}
\end{align}
 where we have defined the symmetrized KL divergence as:
\begin{align}
\dsymm[\beta_{k-1};  \beta_{k}] = D_{KL}^{\,\rightarrow}[\pi_{\beta_{k-1}} ||   \pi_{\beta_{k}}] + D_{KL}^{\,\leftarrow}[\pi_{\beta_{k}} ||   \pi_{\beta_{k-1}}] \nonumber
\end{align}

Our goal in this section will be to show that \eqref{eq:symm_kl_integralA} arises from similar `thermodynamic integration' on the graph of the mean parameters $\eta_{\beta}$.  Recall that we previously applied the fundamental theorem of calculus to $\psi(\beta) = \log Z_{\beta}$ to obtain the difference in log-partition functions
\begin{align}
\psi(\beta_k) - \psi(\beta_{k-1}) = \int \limits_{\beta_{k-1}}^{\beta_k} &\nabla_{\beta} \psi(\beta) d\beta \nonumber
\end{align}
We can obtain a similar expression for the mean parameters $\eta_{\beta} = \nabla_{\beta} \psi(\beta)$ by integrating over the second derivative.
\begin{align}
 \eta_{k} - \eta_{k-1} =   \int \limits_{\beta_{k-1}}^{\beta_k} &\nabla_{\beta}^2 \psi(\beta) d\beta \label{eq:symm_kl_tiA} 
\end{align}

Recalling that $\nabla_{\beta}^2 \psi(\beta) =\Var_{\pibeta} \logiw$, we see that the integrands in ~\cref{eq:symm_kl_integralA} and \cref{eq:symm_kl_tiA} are identical.  Integrating with respect to $\beta$, we obtain the `area of a rectangle' identity for the symmetrized KL divergence (as in \eqref{eq:symm_kl_square}):
\begin{align}
\dsymm [\beta_{k-1};  \beta_{k}] &=  \Delta_{\beta_k} \cdot \int \limits_{\beta_{k-1}}^{\beta_{k}} \Var_{\pibeta} \bigg[ \logiw d\beta \bigg] \nonumber \\
&=(\beta_k - \beta_{k-1}) \int \limits_{\beta_{k-1}}^{\beta_k} \nabla^2_{\beta} \psi(\beta) d\beta \nonumber \\
&= (\beta_k - \beta_{k-1}) \big(  \nabla_{\beta} \psi(\beta) \, \big|_{\beta_{k-1}}^{\beta_k} \big) \nonumber \\
&= (\beta_k - \beta_{k-1})(\eta_{k} - \eta_{k-1}) \label{eq:symm_kl_squareA}
\end{align}
This identity is best understood via Fig. \ref{fig:symm_kl_one_piece} in  Sec. \ref{sec:taylor4}.

To summarize, we have given several equivalent ways of understanding the symmetrized KL divergence.
The `forward' and `reverse' KL divergences arise as gaps in the \gls{TVO} left- and right-Riemann approximations (\cref{fig:symm_kl_one_piece}), or first order Taylor remainders as in \cref{eq:kl_var_lr} and \cref{eq:kl_var_rl}.   Summing these quantities corresponds to the area of a rectangle \cref{eq:symm_kl_squareA} on the graph of the \gls{TVO} integrand $\eta_{\beta}$, or to the integral of a variance term via the Taylor remainder theorem ~\cref{eq:symm_kl_integralA} or fundamental theorem of calculus ~\cref{eq:symm_kl_tiA}.

Note that the \gls{TVO} integrand $\eta_{\beta} = \nabla_{\beta} \psi(\beta) = \mathbb{E}_{\pibeta}[ \logiw]$ will be linear when its derivative, the variance of the log importance weights, is constant within $\beta \in [\beta_{k-1}, \beta_{k}]$.  
 The KL divergence is actually symmetric in this case, which we treat in more detail in the next section (App. \ref{app:fisher_integral}).  More generally, the curvature of the integrand indicates which direction of the KL divergence has larger magnitude, and ~\cref{fig:symm_kl_one_piece} reflects our empirical observations that $D_{KL}[\pi_{\beta_{k-1}} ||   \pi_{\beta_{k}}] > D_{KL}[\pi_{\beta_{k}} ||   \pi_{\beta_{k-1}}].$

\section{Asymptotic Linear Scheduling Analysis} \label{app:fisher_integral}

\citet{grosse2013annealing} treat a quantity identical to $\tvolb$ in the context of analysing the variance of \gls{AIS} estimators.  Using the Central Limit Theorem, \citet{neal2001annealed} show that the variance of an \gls{AIS} estimator is monotonically related to $\tvolb$ under perfect transitions, or independent, exact samples from each intermediate $\beta$ (see \citet{grosse2013annealing} Eq. 3).  However, note that \gls{AIS} estimates expectations over \textit{chains} of $\MCMC$ samples rather than the simple reweighting used in the \gls{TVO}.

In this section, we provide additional perspective on the analysis of \citet{grosse2013annealing}, which considers the asymptotic behavior of the scaled gap in $\tvolb$, $K \cdot D_{\underset{\rightarrow}{KL}}[\pi_{\beta_{k-1}} || \pi_{\beta_{k}}]$, as $K \rightarrow \infty$.

We begin by restating Theorem 1 of \citet{grosse2013annealing} for the case of the full \gls{TVO} objective.  We describe the resulting `coarse-grained' linear binning schedule for choosing $\{ \beta_k \}$ in \ref{app:recursive_scheduling} and provide further analysis in \ref{app:grosse_perspective}.

\begin{theorem}[\citet{grosse2013annealing}]\label{thm:grosse}
Suppose $K+1$ distributions $\{\pi_{\beta_k}\}_{k=0}^{K}$ are linearly spaced along a path $\mathcal{P}$. Under the assumption of perfect transitions, if the Fisher information matrix $G(\beta)$ is smooth, then as $K \rightarrow \infty$:
\begin{align}
    K \sum \limits_{k=1}^{K} D_{KL}^{\,\rightarrow}&[\pi_{\beta_{k-1}}  || \pi_{\beta_{k}}]  \rightarrow  \frac{1}{2} \int \limits_{0}^{1} \dot{\beta}(t) \cdot G \big(\beta(t)\big) \cdot \dot{\beta}(t) dt  \label{eq:path_} \\
    &= \frac{1}{2} \big( D_{KL}^{\,\rightarrow}[\pi_{\beta_0} || \pi_{\beta_K}] + D_{KL}^{\,\leftarrow}[\pi_{\beta_K} || \pi_{\beta_0}] \big) \nonumber 
\end{align}
Here, we let $t \in [0,1]$ parameterize the path $\beta(t) = (1-t)\cdot \beta_0 + t \cdot \beta_K$, and let $\dot{\beta}(t)$ denote the derivative of the parameter $\beta$ with respect to $t$.   For linear mixing of the natural parameters as above, this is a constant: $\dot{\beta}(t) = \beta_K - \beta_0$.   In the case of the full \gls{TVO} integrand, $\dot{\beta}(t) = 1$.
\end{theorem}
\begin{proof}
See \cite{grosse2013annealing} for a detailed proof, which proceeds by taking the Taylor expansion of $D_{KL}[\beta_{k} || \beta_{k} + \Delta_{\beta} ] $ around each $\beta_{k}$ for small $\Delta_{\beta}$.   In particular,  $\Delta_{\beta} = \frac{1}{K} (\beta_K - \beta_0)$ for linearly spaced $\beta_k = (1- \frac{k}{K})\cdot \beta_0 + \frac{k}{K} \cdot \beta_K$.  We assume w.l.o.g. $\beta_K - \beta_0 = 1$ and $\Delta_{\beta} = \frac{1}{K}$ as in \citet{grosse2013annealing} or \gls{TVO}.

The zero- and first-order terms vanish, and the second-order term, with $\Delta_{\beta}^2 = \frac{1}{K^2}$, can be written as (see e.g. \citet{kullback1997information} p. 26):
\begin{align}
K \sum \limits_{k=1}^{K} D_{KL}[\beta_{k} || & \beta_{k} + \Delta_{\beta} ] = K \cdot \frac{1}{2 K^2} \sum \limits_{k=1}^K  \dot{\beta}_k \cdot G (\beta_k) \cdot \dot{\beta}_k \nonumber \\
&\phantom{\beta_{k} + \Delta_{\beta} ] =} + K \cdot \mathcal{O}(K^{-3}) \\
& \rightarrow \frac{1}{2}  \, \int \limits_0^1 \dot{\beta}(t) G\big(\beta (t) \big) \dot{\beta}(t) dt
\end{align}
where we have absorbed $\Delta_{\beta} = \frac{1}{K}$ into a continuous measure $dt$ as $K \rightarrow \infty$.

We now show that this expression corresponds to the symmetrized KL divergence, as in \cite{amari2016information,dabak2002relations}.  While this was not stated in the theorem of \citet{grosse2013annealing}, it has also been shown by e.g. \citet{huszar_2017}.  Observe that $G(\beta) = \nabla_{\beta} \psi(\beta) = \text{Var}_{\pibeta}[T(\vx,\vz)]$ as in \cref{eq:symm_kl_integralA} and \cref{eq:symm_kl_squareA}.  Noting that the chain rule implies $\frac{d}{dt} G \big(\beta (t) \big) = \frac{d}{d\beta} G \big(\beta (t) \big) \frac{d\beta}{dt}$, we can pull one term of $\frac{d\beta}{dt} = \dot{\beta}(t) = (\beta_K - \beta_0)$ outside the integral and perform integration by substitution.  Ignoring the $1/2$ factor,
\begin{align}
\hspace*{-.2cm}\frac{(\beta_K - \beta_0)}{2} \,& \int \limits_0^1 G\big(\beta (t) \big) \frac{d\beta}{dt} dt \nonumber = \frac{(\beta_K - \beta_0)}{2} \int \limits_{\beta_0}^{\beta_K} \nabla_{\beta}^2 \psi(\beta) d\beta \nonumber \\
&= \frac{1}{2} (\beta_K - \beta_0) (\eta_K - \eta_0) \label{eq:symm_kl_square_pf} \\
&= \frac{1}{2} \big( D_{KL}^{\,\rightarrow}[\pi_{\beta_0} || \pi_{\beta_K}] +  D_{KL}^{\,\leftarrow}[\pi_{\beta_K} || \pi_{\beta_0}] \big) \nonumber
\end{align}
\end{proof}

\subsection{`Coarse-Grained' Linear Schedule}\label{app:recursive_scheduling}
\citet{grosse2013annealing} then use this asymptotic condition \cref{eq:symm_kl_square_pf} as $K \rightarrow \infty$  to inform the choice of a \textit{discrete} partition $\mathcal{P} = \{\beta_k\}_{k=0}^{K}$.

More concretely, consider dividing the interval $[0,1]$ into $J$ equally-spaced knot points $\{\beta_j\}_{j=0}^{J}$.  We then allocate a total budget of $K = \sum_{j=1}^J K_j$ intermediate distributions across sub-intervals $[\beta_{j-1}, \beta_j]$, with uniform linear spacing of the $K_j$ partitions within each sub-interval.

Using \cref{eq:symm_kl_square_pf}, \citet{grosse2013annealing} assign a cost $F_j = (\beta_j - \beta_{j-1})(\eta_j - \eta_{j-1}) $ to each `coarse-grained' interval $[\beta_{j-1}, \beta_j]$.  Minimizing $\sum_j F_j$  subject to $\sum_j K_j = K$, the allocation rule becomes:
\begin{align}
    K_j \propto \sqrt{(\beta_{j+1} - \beta_{j}) (\eta_{j+1} - \eta_{j})} \label{eq:opt_bins}
\end{align}

We observe that performance when using this method can be sensitive to the number of knot points used, and we found $J=20$ to perform best in our experiments.

\subsection{Additional Perspectives on \citet{grosse2013annealing}} \label{app:grosse_perspective}
\paragraph{Geometric Intuition for Theorem 1:}
To further understand Theorem 1 of \citet{grosse2013annealing}, observe that the \gls{TVO} integrand will appear linear within any interval $[\beta_{k-1}, \beta_k]$ as $K \rightarrow \infty$.  For general endpoints $\beta_0$ and $\beta_K$, we let $\Delta_{\beta} = \beta_k - \beta_{k-1} = \frac{\beta_K - \beta_0}{K}$.

Having already visualized the symmetrized KL divergence as the area of a rectangle in \cref{fig:symm_kl_one_piece}, we can see that each directed KL divergence, $D_{KL}^{\,\rightarrow}[\pi_{\beta_{k-1}}  || \pi_{\beta_{k}}]$ and $D_{KL}^{\,\leftarrow}[\pi_{\beta_{k}}  || \pi_{\beta_{k-1}}]$, will approach the area of \textit{triangle} as the integrand becomes linear or $K\rightarrow \infty$, with area equal to $1/2 \cdot \Delta_{\beta} \cdot \Delta_{\eta}$ .  Then, the $D_{KL}$ scaled by $K$ becomes 
\begin{align}
K \sum \limits_{k=1}^{K} D_{KL}^{\,\rightarrow}[\pi_{\beta_{k-1}} &|| \pi_{\beta_{k}}] \rightarrow K \sum \limits_{k=1}^{K} \frac{1}{2} \cdot \Delta_{\beta} \cdot \Delta_{\eta} \label{eq:quasi} \\
&= K \sum \limits_{k=1}^{K} \frac{1}{2} (\frac{\beta_K - \beta_0}{K}) \cdot (\eta_k - \eta_{k-1}) \nonumber \\
&= \frac{1}{2} (\beta_K - \beta_0) \cdot (\eta_K - \eta_{0}) \, . \nonumber
\end{align}
where, in the last line, we cancel factors of $K$ and note the cancellation of intermediate $\eta_k$ in the telescoping sum.

\paragraph{Thermodynamic Interpretation:}  This limiting behavior is also discussed in thermodynamics, where the LHS of \cref{eq:quasi} and \cref{eq:path_} corresponds to the rate of entropy production in transitioning a system from $\pi_{\beta_{0}}$ to $\pi_{\beta_{1}}$ along a path defined by $\{\beta_k\}$.  The condition that $K \rightarrow \infty$ refers to the \textit{quasi-static} limit, in which the system is allowed to remain at equilibrium throughout the process (e.g. \citet{crooks2007measuring}).

\paragraph{Exponential and Mixture Geodesics:}
As in the statement of Theorem 1, we can more generally consider connecting two distributions, indexed by natural parameters $\beta_0$ and $\beta_1$, using a parameter $t \in [0,1]$.  The curve $\beta_t = (1-t) \cdot \beta_0 + t \cdot \beta_1$ then corresponds to our path exponential family \cref{eq:exp_fam}, and is also referred to as the $e$-geodesic in information geometry \citet{amari2016information}.

Similarly, the moment-averaged path of \citet{grosse2013annealing}, which also underlies our scheduling strategy in Sec. \ref{sec:schedules}, can be viewed as a linear mixture in the mean parameter space.   The $m$-geodesic then refers to the curve $\eta_t = (1-t) \cdot \eta_0 + t \cdot \eta_1$ \cite{amari2016information}.  Note that these mixtures reference different distribution for the same parameter $t$, so that $\eta_{\beta_t} \neq \eta_t$.

\citet{grosse2013annealing} proceed to show that the expression for the symmetric KL divergence \eqref{eq:path_} corresponds to the integral of the Fisher information along \textit{either} the geometric or mixture paths (Theorem 2 of \citet{grosse2013annealing}, Theorem 3.2 of \citet{amari2016information}).   The union of the intermediate distributions integrated by these two paths coincide in our one-dimensional exponential family, although this intuition does not appear to translate to higher dimensions.  



\section{Higher Order \gls{TVO}} \label{app:higher_order_tvo}
While the convexity of the log-partition function yields the family of Bregman divergences from the remainder in the first order Taylor approximation, we might also consider higher order terms to obtain tighter bounds on likelihood or analyse properties of the TVO integrand $\nabla_{\beta} \psi(\beta)$.  We give an example derivation for a second-order \gls{TVO} objectives, although these are no longer guaranteed to be upper or lower bounds on likelihood.

\paragraph{Left-to-Right Expansion}
We first consider expanding the approximations in the \gls{TVO} left-Riemann sum to second order.  We denote the resulting objective $\mathcal{L}_{\rightarrow}^{(2)}$, since we move `left-to-right' in estimating $\psi(\beta_k)$ around $\beta_{k-1}$.  We begin by writing the second-order Taylor approximation:
\begin{align}
\psi(\beta_k) & \approx \psi(\beta_{k-1}) + (\beta_{k} - \beta_{k-1}) \, \nabla_{\beta} \psi(\beta_{k-1}) \nonumber \\
   & \phantom{\approx}+  \frac{1}{2} (\beta_{k} - \beta_{k-1})^2 \, \nabla_{\beta}^2 \psi(\beta_{k-1}) \label{eq:second_order_lr}
\end{align}
While $\tvolb$ consists of the first-order term alone, we can also consider adding the non-negative, second-order term to form the objective $\mathcal{L}_{\rightarrow}^{(2)}$.  Using successive Taylor approximations of $\psi(\beta_k)$, we obtain similar telescoping cancellations to obtain
\begin{align} 
\log  p(\vx) - \mathcal{L}_{\rightarrow}^{(2)} &= \log p(\vx) - \sum \limits_{k=1}^{K} (\beta_{k} - \beta_{k-1}) \cdot \eta_{\beta_{k-1}} \nonumber \\
\phantom{=} &- \sum \limits_{k=1}^{K} \frac{1}{2} (\beta_{k} - \beta_{k-1})^2 \,\text{Var}_{\pi_{\beta_{k-1}}} \log \frac{p(\vx,\vz)}{q(\vz|\vx)}  \label{eq:second_order_lb}
\end{align}
where $\eta_{\beta_{k-1}} = \mathbb{E}_{\pi_{\beta_{k-1}}} \log \frac{p(\vx,\vz)}{q(\vz|\vx)}$.

We previously obtained a lower bound on log-likelihood via this construction, with $\log  p(\vx) - \tvolb \geq 0$.  However, $\mathcal{L}_{\rightarrow}^{(2)}$ will only provide a lower bound if $\nabla_{\beta} \psi(\beta)$ is concave, i.e.  $\nabla^3_{\beta} \psi(\beta) \leq 0.$  To see this, we write the Taylor remainder \cref{eq:remainder1} as
\begin{align}
   R_{2}^{\rightarrow}(\beta_k)  &= \int \limits_{\beta_{k-1}}^{\beta_k} \frac{1}{6} (\beta_{k} - \beta_{t})^3 \, \nabla_{\beta}^3 \psi(\beta_{t}) d\beta_t
\end{align}
with the third derivative equal to
\begin{align*}
   \nabla_{\beta}^3 \psi(\beta) &= \mathbb{E}_{\pi_{\beta}} [T(\vx, \vz)^3] - 3 \, [\mathbb{E}_{\pi_{\beta}} T(\vx, \vz)] \cdot \mathbb{E}_{\pi_{\beta}} [T(\vx, \vz)^2] \\
   & \phantom{=} + 2  [\mathbb{E}_{\pi_{\beta}} T(\vx, \vz)]^3 \\ \\
  &= \mathbb{E}_{\pi_{\beta}} [T(\vx, \vz)^3] - [\mathbb{E}_{\pi_{\beta}} T(\vx, \vz)]^3 \\
  &\phantom{=} - 3 [\mathbb{E}_{\pi_{\beta}}T(\vx, \vz)] \cdot [\text{Var}_{\pi_{\beta}} T(\vx, \vz) ]
\end{align*}
In addition to indicating that $\mathcal{L}_{\rightarrow}^{(2)}$ is a lower bound on $\log \px$, testing the concavity of $\nabla_{\beta} \psi(\beta)$ using $\nabla^3_{\beta} \psi(\beta) \leq 0$ can also indicate whether a trapezoid approximation to the \gls{TVO} integral provides a valid lower bound.

We can give an identical construction for the reverse direction $\mathcal{L}_{\leftarrow}^{(2)}$ or higher order approximations.  We leave a full exploration of these objectives for future work.

\section{Experimental Setup}\label{app:experiments}
Code for all experiments can be found at \url{https://github.com/vmasrani/tvo_all_in}. 
\paragraph{Model}
Following \cite{Burda2015}, we use a variational autoencoder~\cite{Kingma2013} with a 50-dimensional stochastic layer, $\vz \in \R^{50}$
\begin{align*}
    p_{\theta}(\vx, \vz) &= p_{\theta}(\vx|\vz)p(\vz)\\
    p(\vz) &= \N(\vz | 0, \vI)\\
    p_{\theta}(\vx|\vz) &= \Bern(\vx | \text{decoder}_{\theta}(\vz))\\
    q_{\phi}(\vz|\vx) &= \N(\vz ; \boldsymbol{\mu}_{\phi}(\vx), \boldsymbol{\sigma}_{\phi}(\vx)) 
\end{align*}
where the encoder and decoder are each two-layer MLPs with tanh activations and 200 hidden dimensions. The output of the encoder is duplicated and passed through an additional linear layer to parameterize the mean and log-standard deviation of a conditionally independent Normal distribution. The output of the decoder is a sigmoid which parameterizes the probabilities of the independent Bernoulli distribution. $\theta$ and $\phi$ refer to the weights of the decoder and encoder, respectively.
\paragraph{Dataset}
We use Omniglot~\cite{lake_one-shot_2013}, a dataset of 1623 handwritten characters across 50 alphabets. Each datapoint is binarized $28 \times 28$ image, i.e $\vx \in \{0,1\}^{784}$, where we follow the common procedure in the literature of sampling each binary-valued observation with expectation equal to the real pixel value~\cite{salakhutdinov2008quantitative,Burda2015}. We split the dataset into 24,345 training and 8,070 test examples.
\paragraph{Training Procedure}
All models are written in PyTorch and trained on GPUs. For each scheduler, we train for 5000 epochs using the Adam optimizer~\cite{kingmaAdamMethodStochastic2017} with a learning rate of $10^{-3}$, and minibatch size of 1000. All weights are initialized with PyTorch's default initializer.

\section{Implementation Details}\label{app:implementation}
While the Legendre transform, mapping between a target value of expected sufficient statistics $\eta = \mathbb{E}_{\pibeta}[T(\omega)]$ and the appropriate natural parameters $\beta$, can be a difficult problem in general, we describe how to efficiently implement our `moments-spacing' schedule in the context of \gls{TVO}.

Recall from Sec. \ref{sec:schedules} that we are interested in finding a discrete partition $\mathcal{P}_{\beta} = \{ \beta_k \}_{k=0}^{K}$ such that:
\begin{align}
\beta_k = \eta^{-1}_{\beta}\left( \frac{k}{K} \cdot \gls{ELBO} + (1-\frac{k}{K} ) \cdot \gls{EUBO} \right) \label{eq:moments_scheduleA}
\end{align}
In other words, we seek to find the $\beta_k$ such that $\mathbb{E}_{\pi_{\beta_k}} [\logiw] \approx \eta_k$, where $\eta_k$ are equally spaced between the \gls{ELBO} and \gls{EUBO} (see \cref{fig:moment_schedule}).


\begin{figure*}[t]
    \caption{Pseudo-code Implementation of Moments Scheduling for \gls{TVO}}\label{fig:moments_code}
\begin{minipage}{.5\textwidth}
\begin{lstlisting}[language=Python]
# Calculate expected sufficient statistics eta at a given beta (Eq. 12)
def calc_eta(log_iw, beta):
    # 1) Exponentiate/normalize over importance sample dimension
    snis = torch.exp(log_iw*beta -
        torch.logsumexp(log_iw*beta,
          dim = 1, keepdim=True))
    # 2) Take mean over data examples
    return torch.mean(snis*log_iw, dim =0)

def binary_search(target, log_iw, start=0, stop=1, threshold = 0.1):

    beta_guess = .5*(stop-start)
    eta_guess = calc_eta(log_iw,beta_guess)
    if eta_guess > target + threshold:
        return binary_search(
          target,
          log_iw,
          start=beta_guess,
          stop=stop)
    elif eta_guess < target - threshold:
        return binary_search(
          target,
          log_iw,
          start=start,
          stop=beta_guess)
    else:
        return beta_guess
\end{lstlisting}
\end{minipage}%
\begin{minipage}{.5\textwidth}
\begin{lstlisting}[language=Python]

def moments_spacing_schedule(log_iw, K, search='binary'):
    # 1) Calculate target values for uniform moments spacing
    elbo = calc_eta(log_iw, 0)
    eubo = calc_eta(log_iw, 1)
    targets = [(1-t)*elbo+t*eubo
      for t in np.linspace(0,1,K+1)]

    # 2) Find beta corresponding to each target (including beta=0,1)
    beta_schedule = [0]

    for _k in range(1, K):
        target_eta = targets[_k]

        beta_k = binary_search(
                    target_eta,
                    log_iw,
                    start = 0,
                    stop = 1)

        beta_schedule.append(beta_k)

    beta_schedule.append(1)
    # 3) Return beta_schedule: used for Riemann approximation points in TVO objective
    return beta_schedule
\end{lstlisting}
\end{minipage}%
\end{figure*}

More concretely, we provide pseudo-code implementing our moments spacing schedule below.  Given a set of $S$ log-importance weights per sample, and a number of intermediate distributions $K$:

\section{Additional Results}\label{app:addl_results}

In this section, we report wall-clock runtimes and run similar experiments as in Sec.~\ref{sec:experiments} to evaluate our moments spacing schedule and reparameterized gradients on the binarized MNIST dataset \cite{salakhutdinov2008quantitative}.

\paragraph{Wall-Clock Times}
We report wall clock runtimes for various scheduling methods with $S=50$ and $K=5$ in Fig. \ref{fig:wallclock}.  While \gls{TVO} methods require slight overhead compared with $\gls{IWAE}$, our adaptive moments scheduler does not require significantly more computation than the log-uniform baseline.

\paragraph{Grid Search Comparison} 
    We evaluate our moments schedule with $K=2$ against grid search over the choice of a single intermediate $\beta_1$ in Fig. \ref{fig:single_beta_mnist}.  The setup is similar to that of Fig. \ref{fig:single_beta_omni} on Omniglot (see Sec. \ref{sec:experiments}), but here we use reparameterization gradients instead of the original \gls{TVO}.  Here, we train for 1000 epochs using an Adam optimizer with learning rate $10^{-3}$ and batch size 100.

   We again find that our moments spacing schedule arrives at an optimal choice of $\beta_1$, and can even outperform the best static value due to its ability to adaptively update at each epoch.
   It is interesting to note that the final choice of $\beta_1$, which reflects the shape of the \gls{TVO} integrand, is nearly identical at $\beta_1 \approx 0.30$ across both \textsc{mnist} and Omniglot.

\paragraph{Comparison with IWAE}
We compare \gls{TVO} using our moments scheduling against the \gls{IWAE} and \gls{IWAE} \textsc{dreg} as in Fig. \ref{fig:iwaes} of the main text.   We find that our \gls{TVO} reparameterized gradient estimator achieves nearly identical model learning performance as \gls{IWAE} and \gls{IWAE} \textsc{dreg}, with notably improved posterior inference for all values of $S$.

\paragraph{Evaluating Scheduling Strategies}
In the Fig. \ref{fig:by_k_mnist}-\ref{fig:by_k_omni_rep} below, we reproduce the setting of Fig. \ref{fig:omni_schedules} to evaluate our scheduling strategies by $K$, for \gls{TVO} with both $\REINFORCE$ and reparameterized gradient estimators, on Omniglot and \textsc{mnist}.  We also report posterior inference results as measured by test $D_{KL}[\qzx||\pzx]$.   In general, we find comparable performance between our moments schedule and the log-uniform baseline, although our approach performs best with $K=2$ and does not require grid search.  Further, on \textsc{mnist} with batch size 1000 and low $K$, log-uniform, linear, and coarse-grained schedules suffer from poor performance due to instability in training, which is avoided by our moments schedule.   Training can be stabilized by using smaller batch sizes as in Fig. \ref{fig:single_beta_mnist}.

\begin{figure*}[h]
    \begin{minipage}{.45\textwidth}
        \includegraphics[width=.9\textwidth]{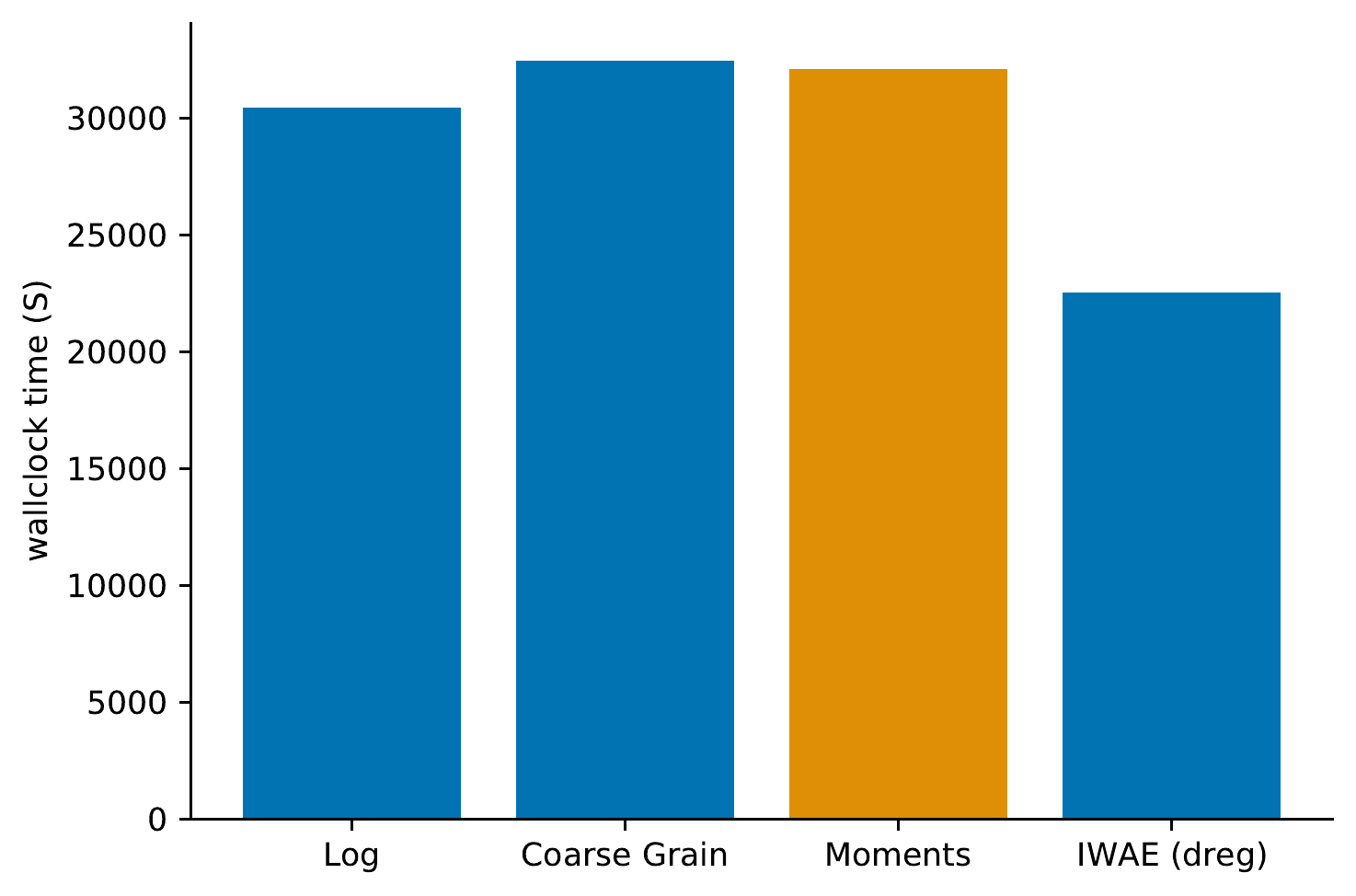}
        \captionof{figure}{Omniglot Runtimes ($S=50$, $K=5$, 5k epochs)} \label{fig:wallclock}
    \end{minipage}\hspace*{.05\textwidth}
    \begin{minipage}{.45\textwidth}
        \includegraphics[width=.9\textwidth]{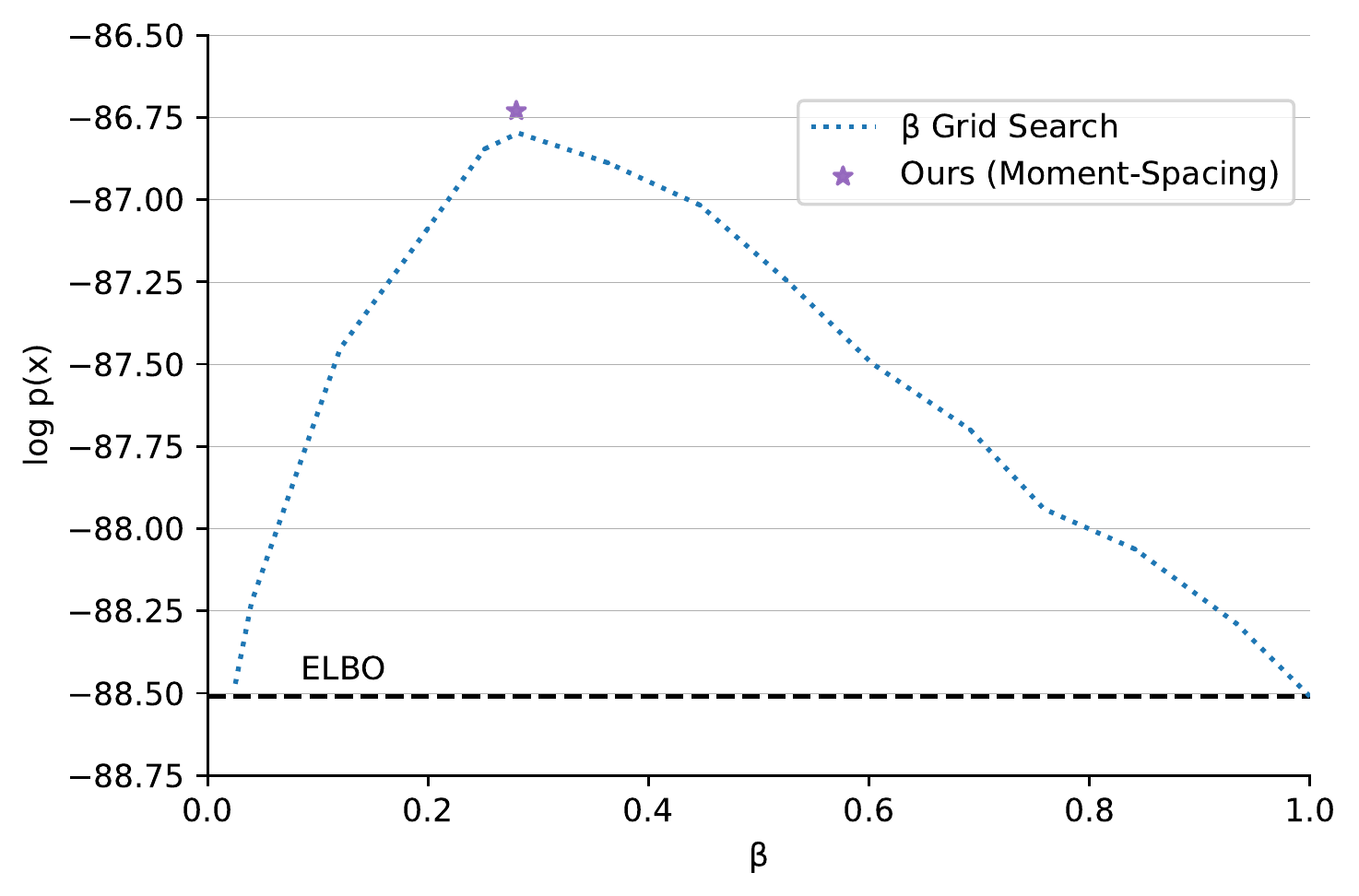}
        \captionof{figure}{MNIST $K=2$, with reparameterization gradients.} \label{fig:single_beta_mnist}
    \end{minipage}%
\end{figure*}

\begin{figure*}[t]
    \vspace*{.25cm}
    \begin{minipage}{.5\textwidth}
        \subcaptionbox{MNIST Test $\log \px$ }{
            \includegraphics[width=.8\textwidth]{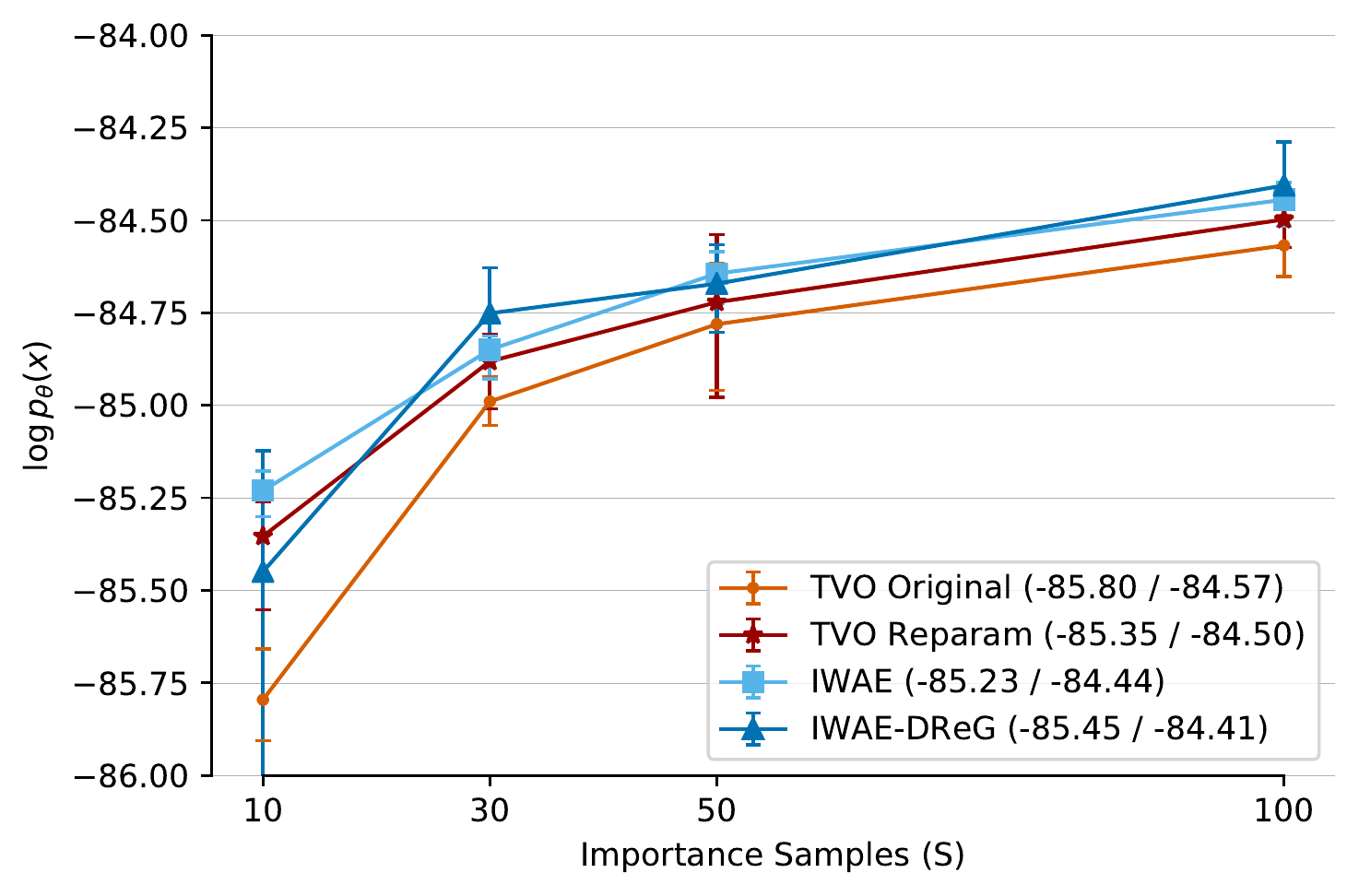}
            }\label{fig:mnist_tvo_schedules2}
        \end{minipage}%
    \begin{minipage}{.5\textwidth}
        \subcaptionbox{MNIST Test $D_{KL}[\qzx||\pzx]$}{
            \includegraphics[width=.8\textwidth]{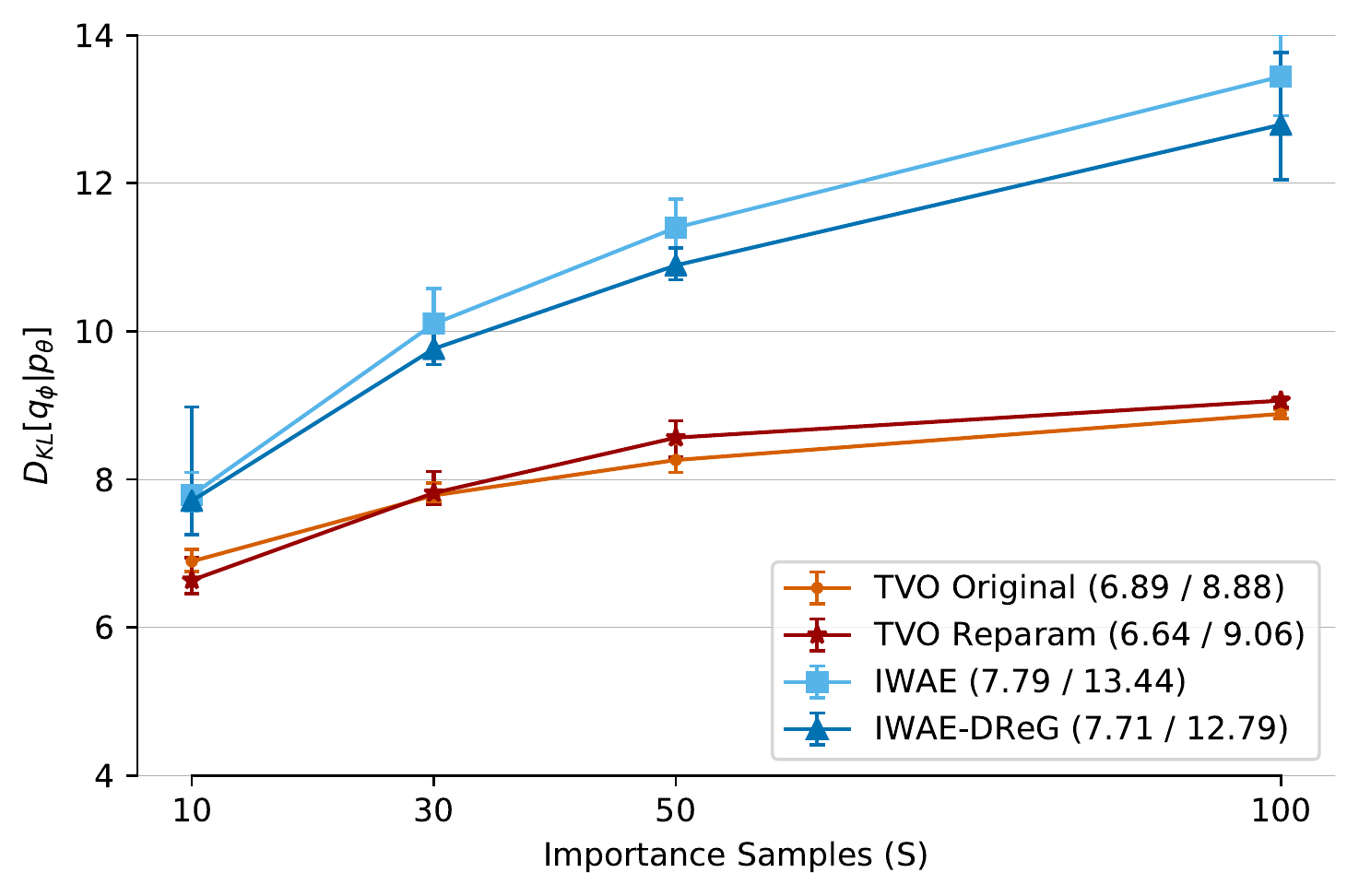}
            }\label{fig:mnist_reparam_schedules2}
        \end{minipage}%
    \caption{Model Learning and Inference by $S$ (with $K=5$)} \label{fig:iwaes_mnist}
\end{figure*}

\begin{figure*}[htb]
    \vspace*{.1cm}
    \begin{minipage}{.5\textwidth}
        \subcaptionbox{MNIST Test $\log \px$ }{
            \includegraphics[width=.85\textwidth]{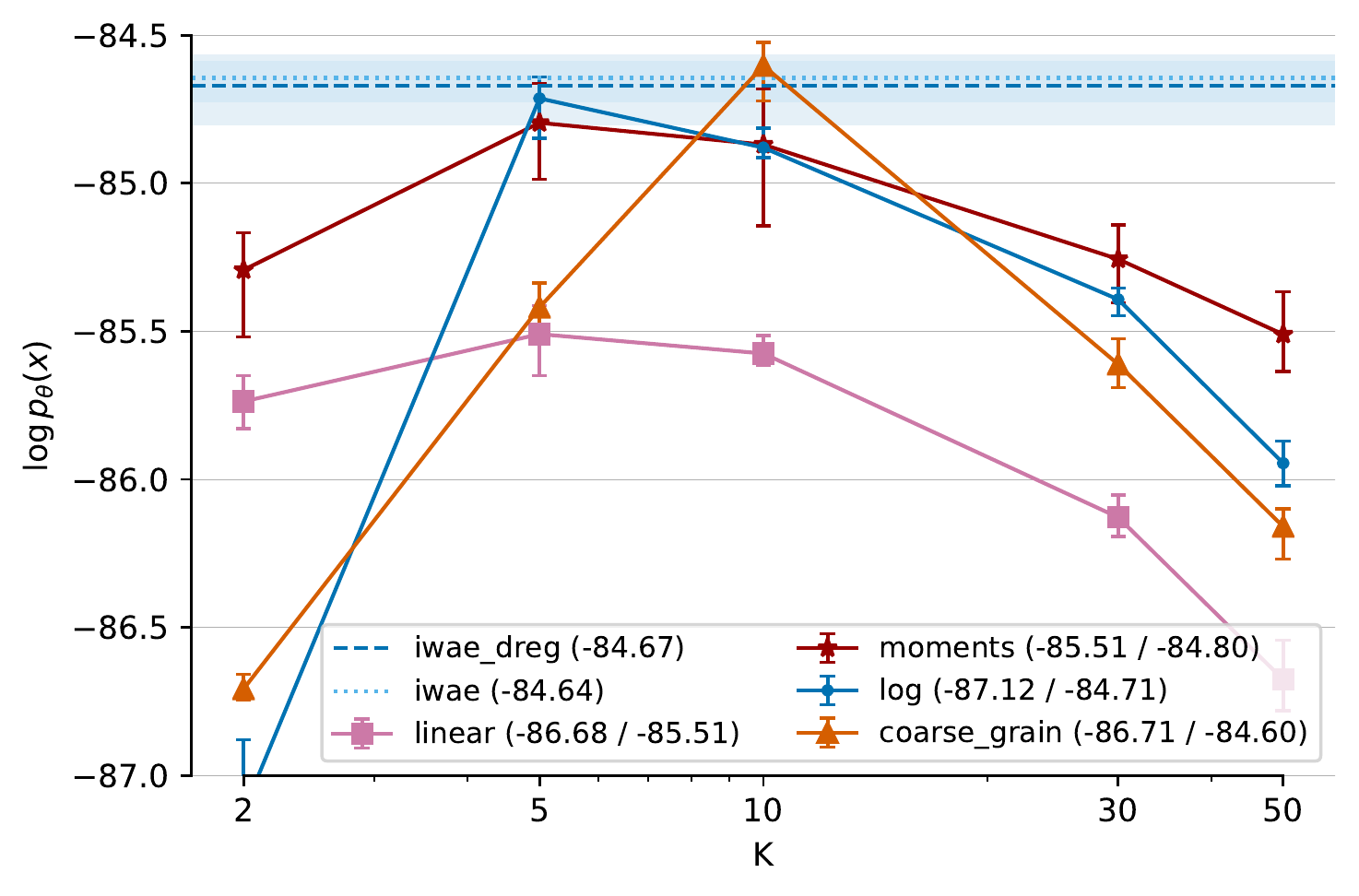}
            }\label{fig:mnist_tvo_schedules2}
        \end{minipage}%
    \begin{minipage}{.5\textwidth}
        \subcaptionbox{MNIST Test $D_{KL}[\qzx||\pzx]$}{
            \includegraphics[width=.85\textwidth]{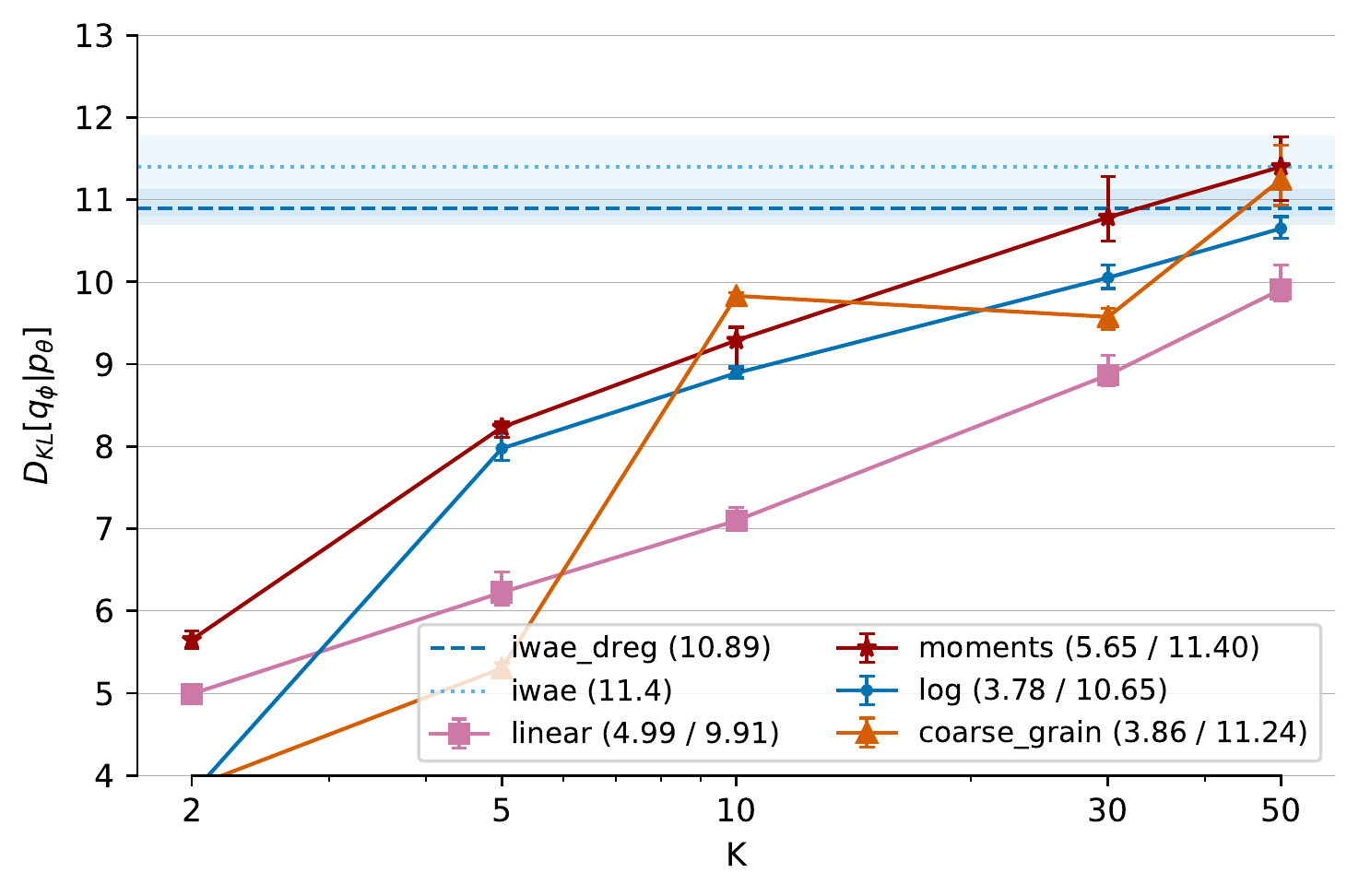}
            }\label{fig:mnist_reparam_schedules2}
        \end{minipage}%
    \caption{\centering \gls{TVO} with $\REINFORCE$ Gradients:
    Model Learning and Inference by $K$ (with $S=50$)} \label{fig:by_k_mnist}
\end{figure*}%
\begin{figure*}[htb]
    \vspace*{.1cm}
    \begin{minipage}{.5\textwidth}
        \subcaptionbox{MNIST Test $\log \px$ }{
            \includegraphics[width=.85\textwidth]{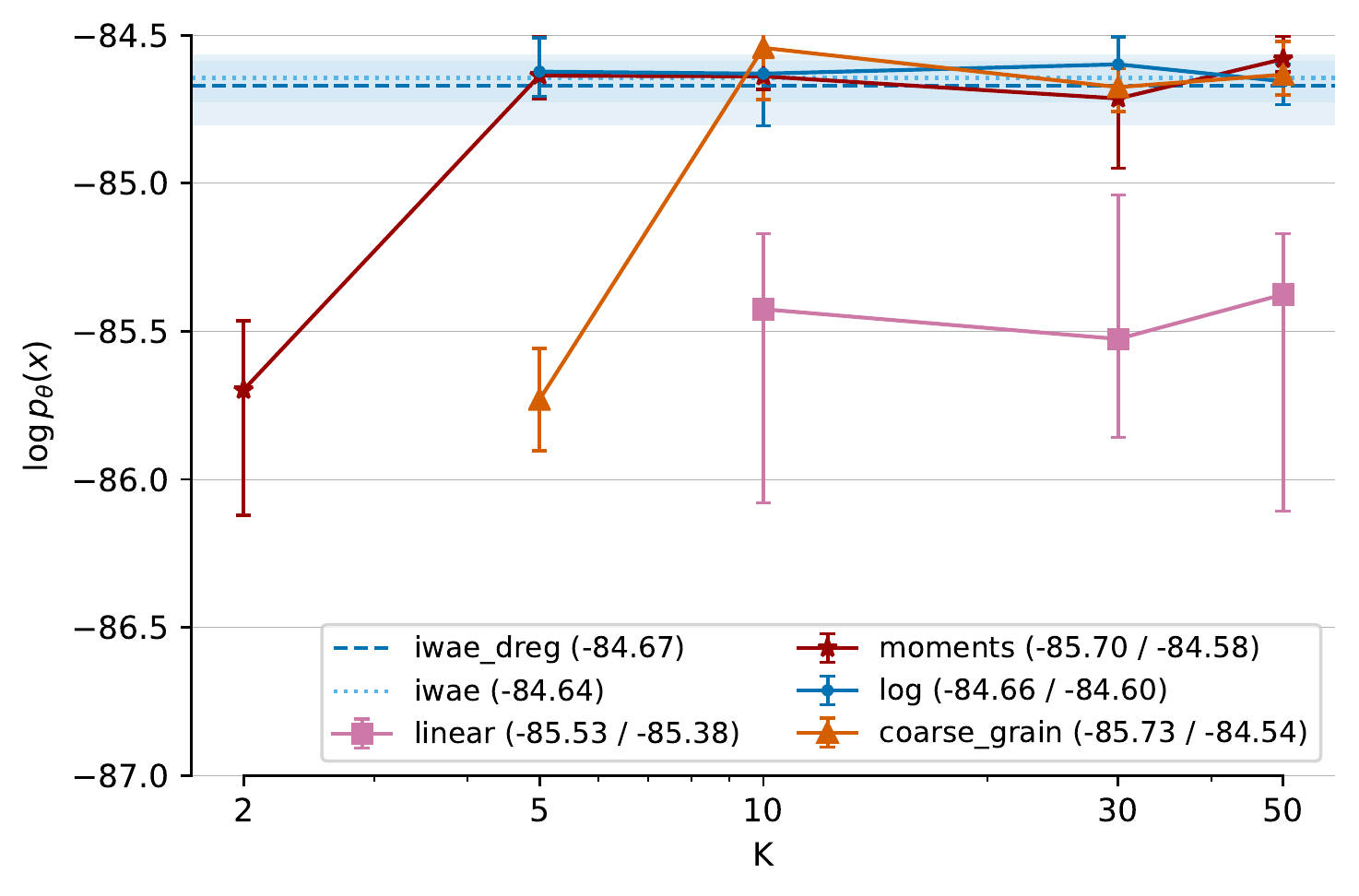}
            }\label{fig:mnist_tvo_schedules2}
        \end{minipage}%
        \begin{minipage}{.5\textwidth}
            \subcaptionbox{MNIST Test $D_{KL}[\qzx||\pzx]$}{
                \includegraphics[width=.85\textwidth]{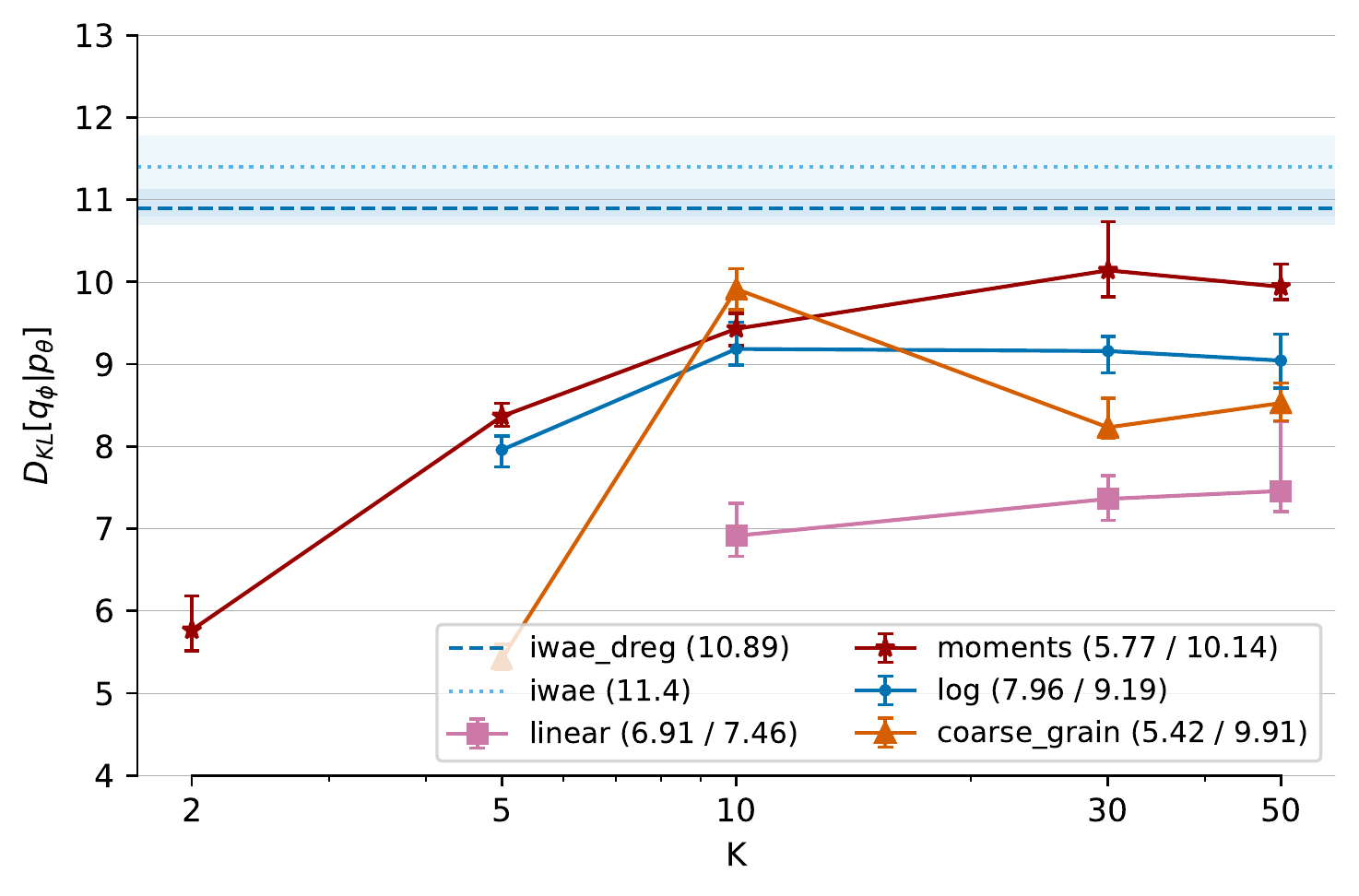}
                }\label{fig:mnist_reparam_schedules2}
            \end{minipage}%
    \caption{\centering \gls{TVO} with Reparameterized Gradients:
    Model Learning and Inference by $K$ (with $S=50$)} \label{fig:by_k_mnist_rep}
\end{figure*}%
\begin{figure*}[htb]
    \vspace*{.1cm}
    \begin{minipage}{.5\textwidth}
        \subcaptionbox{Omniglot Test $\log \px$ }{
            \includegraphics[width=.85\textwidth]{figs/by_k_omniglot_tvo_test_logpx.pdf}
            }\label{fig:mnist_tvo_schedules2}
        \end{minipage}%
        \begin{minipage}{.5\textwidth}
            \subcaptionbox{Omniglot Test $D_{KL}[\qzx||\pzx]$}{
                \includegraphics[width=.85\textwidth]{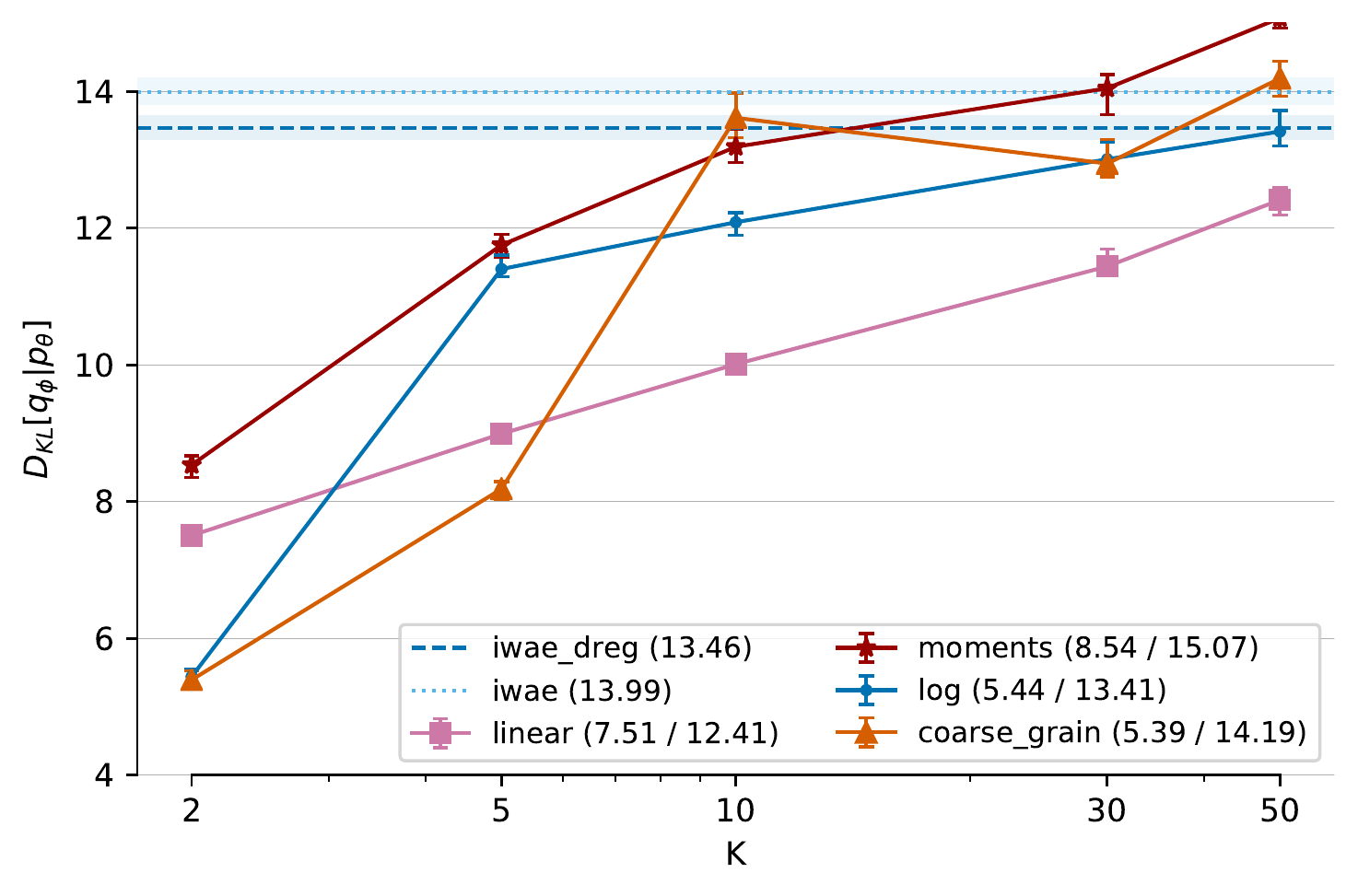}
                }\label{fig:mnist_reparam_schedules2}
            \end{minipage}%
    \captionof{figure}{\centering \gls{TVO} with $\REINFORCE$ Gradients:
    Model Learning and Inference by $K$ (with $S=50$)} \label{fig:by_k_omni}
\end{figure*}%
\begin{figure*}[htb]
    \vspace*{.1cm}
    \begin{minipage}{.5\textwidth}
        \subcaptionbox{Omniglot Test $\log \px$ }{
            \includegraphics[width=.85\textwidth]{figs/by_k_omniglot_tvo_reparam_q_tvo_test_logpx.pdf}
            }\label{fig:mnist_tvo_schedules2}
        \end{minipage}%
        \begin{minipage}{.5\textwidth}
            \subcaptionbox{Omniglot Test $D_{KL}[\qzx||\pzx]$}{
                \includegraphics[width=.85\textwidth]{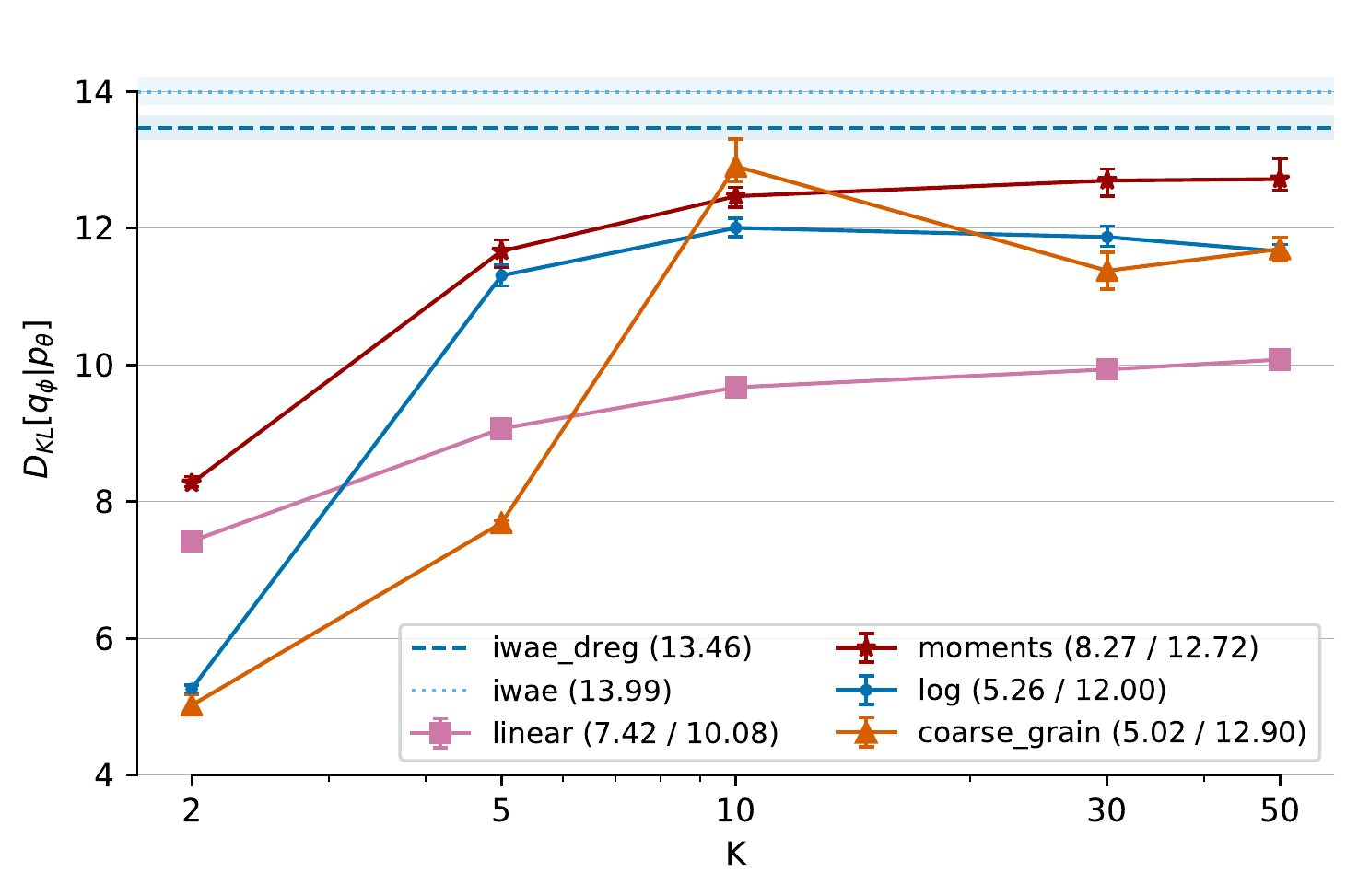}
                }\label{fig:mnist_reparam_schedules2}
            \end{minipage}%
    \caption{\centering \gls{TVO} with Reparameterized Gradients:
    Model Learning and Inference by $K$ (with $S=50$)} \label{fig:by_k_omni_rep}
\end{figure*}

\onecolumn

\newcommand{\pitilde}{{\tilde{\pi}_{\beta}(\vx, \vz)}}
\newcommand{\lpitilde}{\log \pitilde}
\newcommand{\lw}{{\log w_\phi}}
\newcommand{\g}{f(\vz)}
\newcommand{\grep}{f(\vz)}
\newcommand{\total}{\frac{d }{d \phi} }
\newcommand{\reparami}{ \vz_i = z(\epsilon_i, \phi) }
\newcommand{\reparam}{ \vz = z(\epsilon, \phi) }
\newcommand{\dphi}{{\frac{d}{d \phi}}}
\newcommand{\pphi}{{\frac{\partial}{\partial \phi}}}

\newcommand{\Epi}{{\mathbb{E}_{\pi_\beta}}}
\newcommand{\Eq}{\mathbb{E}_{\qzx}}
\newcommand{\Ephi}{\mathbb{E}_{q_{\phi}}}
\newcommand{\Eeps}{\mathbb{E}_{\epsilon}}
\newcommand{\Covpi}{\Cov_{\pi_\beta}}

\newcommand{\circleone}{\raisebox{.5pt}{\textcircled{\raisebox{-.9pt} {1}}}}
\newcommand{\circletwo}{\raisebox{.5pt}{\textcircled{\raisebox{-.9pt} {2}}}}

\newcommand{\pgpz}{{\frac{\partial \grep}{\partial \vz}}}
\newcommand{\pfpphi}{{\frac{\partial \grep}{\partial \phi}}}
\newcommand{\pfpphirep}{{\frac{\partial \grep}{\partial \phi}}}
\newcommand{\pzpphi}{{\frac{\partial \vz}{\partial \phi}}}
\newcommand{\change}[1]{\textcolor{blue}{#1}}

\onecolumn

\section{Reparameterization Gradients for the TVO Integrand}\label{app:tvo_integrand_gradients}
Recall that the \gls{TVO} objective involves terms of the form
\begin{align}
    \mathbb{E}_{\pibeta}  \left[ f(\vz) \right] \quad \text{where} \quad \pibeta(\vz|\vx) = \frac{\qzx^{1-\beta} \pxz^{\beta}}{Z_{\beta}} \quad \text{and} \quad f(\vz) = \log \frac{\pxz}{\qzx} \label{eq:tvo_integrand}
\end{align}
While \citet{masrani2019thermodynamic} derive a $\REINFORCE$-style gradient estimator for the \gls{TVO}, we seek to apply the reparameterization trick when possible, and thus differentiate with respect to only the inference network parameters $\phi$.  Note that, for  $\vz_i \sim \qzx$ reparameterizable with $\reparam$ and $\epsilon_i \sim p(\epsilon)$, any expectation under $\pi_\beta$ can be written as
    \begin{align}
    \Epi \left[ \grep \right] = \frac{1}{Z_\beta}\Eq \left[ w^\beta \grep \right] = \frac{1}{Z_\beta}\Eeps \left[ w^\beta \grep \right] \quad \text{where} \quad   w = \frac{\pxz}{\qzx} \label{eq:appendix/grad/reparam}
\end{align}
In differentiating \cref{eq:appendix/grad/reparam}, we will frequently encounter terms of the form $\mathbb{E}_{\pibeta} \left[ \g \total \log \frac{\pxz}{\qzx}  \right]$ for generic $\g$.  Noting that the total derivative contains score function partial derivatives, we apply the reparameterization trick to these terms in an approach similar to the `doubly-reparameterized' estimator of \citet{tucker2018doubly}.  The following lemma summarizes these calculations, rewritten using expectations under $\pi_{\beta}$ as in \cref{eq:appendix/grad/reparam}.
\begin{lemma}\label{eq:lemma1}
    Let $f(\vz):\mathbb{R}^M \mapsto \mathbb{R}$, $\pibeta(\vz|\vx)$, and $w = \frac{\pxz}{\qzx}$ all depend on $\phi$. When $\vz \sim \qzx$ is reparameterizable via $\reparam, \, \epsilon \sim p(\epsilon)$, the following identity holds for expectations under $\pibeta$
\begin{align}
    \mathbb{E}_{\pibeta} \left[ \grep \frac{d }{d \phi} \log w  \right] = \mathbb{E}_{\pi_\beta}\left[\pzpphi \left((1 - \beta)\grep \frac{\partial\log w}{\partial \vz} - \pgpz \right)\right].
\end{align}
\begin{proof}
    See \cref{sec:proof_lemma_1}.
\end{proof}
\begin{corollary}\label{eq:corr1}
    For the choice of $\g = 1$ we obtain
    \begin{align}
        \mathbb{E}_{\pibeta} \left[\frac{d }{d \phi} \log w  \right] = (1 - \beta) \, \mathbb{E}_{\pi_\beta}\left[\pzpphi \frac{\partial\log w}{\partial \vz}\right]
    \end{align}
\end{corollary}
\end{lemma}
The following lemma will allow us to apply reparameterization within the normalization constant.
\begin{lemma}\label{lemma2}
    Let the same conditions hold as in \cref{eq:lemma1}, with $Z_{\beta} = \int \qzx^{1-\beta} p_{\theta}(\vx,\vz)^{\beta} \vdz$. Then
\begin{align}
    \frac{d }{d \phi} Z_{\beta} = \beta (1 - \beta) \, \mathbb{E}_{\epsilon} \left[w^{\beta} \pzpphi\frac{\partial\log w}{\partial \vz} \right].
\end{align}
\begin{proof}
    See \cref{sec:proof_lemma_2}.
\end{proof}
\end{lemma}

We now proceed to differentiate the TVO integrand given by \cref{eq:tvo_integrand}.

\subsection{Reparameterized \gls{TVO} Gradient Estimator}
For generic $f(\vz):\mathbb{R}^M \mapsto \mathbb{R}$ and reparameterizable $\vz \sim \qzx$ as above, the gradient with respect to $\phi$ can be written as
\begin{align}
\frac{d }{d \phi} \mathbb{E}_{\pibeta}\left[ f(\vz) \right] &= \mathbb{E}_{\pibeta} \left[ \left(\frac{d }{d \phi} f(\vz) \right) - \beta \left(\frac{\partial \vz}{\partial \phi} \frac{\partial f(\vz) }{\partial \vz} \right) \right] + \beta (1-\beta) \text{Cov}_{\pi_{\beta}}  \left[  f(\vz),  \frac{\partial \vz}{\partial \phi} \frac{\partial \log w }{\partial \vz} \right]\label{eq:reparam_generic}.
\end{align}
The gradient of the \gls{TVO} integrand is of particular interest.   For $f(\vz) = \log w$ with $w = \frac{\pxz}{\qzx}$, \cref{eq:reparam_generic} simplifies to
\begin{align}
    \frac{d }{d \phi} \mathbb{E}_{\pibeta}\left[ \log w \right] &= (1-2\beta)\mathbb{E}_{\pibeta} \left[\frac{\partial \vz}{\partial \phi} \frac{\partial \log w }{\partial \vz} \right] + \beta (1-\beta) \text{Cov}_{\pi_{\beta}}  \left[ \log w,  \frac{\partial \vz}{\partial \phi} \frac{\partial \log w }{\partial \vz} \right]\label{eq:reparam_tvo}.
\end{align}

\begin{proof}  We track changes between lines in \change{blue}, and begin by applying the product rule.
\begin{align}
\frac{d }{d \phi} \mathbb{E}_{\pibeta}\left[ \grep \right] &= \frac{d }{d \phi} \left( Z_{\beta}^{-1} \mathbb{E}_{\epsilon}\left[ w^{\beta} \, f(\vz)\right] \right)  \\
 &= \change{\left( \frac{d}{d \phi} Z_{\beta}^{-1} \right)} \mathbb{E}_{\epsilon}\left[ w^{\beta} \, f(\vz)\right] + Z_{\beta}^{-1} \mathbb{E}_{\epsilon}\left[f(\vz) \change{ \left( \frac{d }{d \phi} w^{\beta} \right)}\right] + Z_{\beta}^{-1} \mathbb{E}_{\epsilon}\left[ w^{\beta} \change{ \left( \frac{d }{d \phi} f(\vz) \right)}\right] \\
 &= \change{\left(\frac{d }{d \phi}  Z_{\beta}\right)\left( \frac{-1}{Z_{\beta}^{2}}  \right) }\mathbb{E}_{\epsilon}\left[ w^{\beta} \, f(\vz)\right] + Z_{\beta}^{-1} \mathbb{E}_{\epsilon}\left[\change{\beta w^{\beta} \, f(\vz)  \left( \frac{d }{d \phi} \log w\right)}\right] + Z_{\beta}^{-1} \mathbb{E}_{\epsilon}\left[ w^{\beta} \left( \frac{d }{d \phi} f(\vz) \right)\right] \\
 &=\underbrace{\left(\frac{d }{d \phi}  Z_{\beta}\right)\change{\left( \frac{-1}{Z_{\beta}}  \right) }\mathbb{E}_{\change{\pibeta}} \left[f(\vz)\right]}_{\circleone}
 + \underbrace{\change{\beta} \, \mathbb{E}_{\change{\pibeta}}\left[f(\vz)  \frac{d }{d \phi} \log w\right]}_{\circletwo}
 + \mathbb{E}_{\change{\pibeta}} \left[ \frac{d }{d \phi} f(\vz) \right] \label{eq:pre_cov_sub}
\end{align}
We proceed to simplify only the first two terms, applying \cref{lemma2} to $\circleone$ and \cref{eq:lemma1} to $\circletwo$.
\begin{align}
\circleone + \circletwo = &\underbrace{\change{\beta (1 - \beta)\mathbb{E}_{\epsilon} \left[w^{\beta} \pzpphi\frac{\partial\log w}{\partial \vz} \right]}}_{\text{\cref{lemma2}}}\left( \frac{-1}{Z_{\beta}}  \right) \mathbb{E}_{\pibeta} \left[f(\vz)\right] + \beta\underbrace{\change{ \left((1 - \beta)\mathbb{E}_{\pi_\beta}\left[\pzpphi \frac{\partial\log w}{\partial \vz} \grep \right] - \mathbb{E}_{\pi_\beta}\left[\pzpphi \pgpz \right]\right)}}_{\text{\cref{eq:lemma1}}}\\
= &\beta (1 - \beta)\mathbb{E}_{\change{\pibeta}} \left[\pzpphi\frac{\partial\log w}{\partial \vz} \right] \change{\left(-1\right)}\mathbb{E}_{\pibeta} \left[f(\vz)\right]
\change{+ \beta(1 - \beta )} \mathbb{E}_{\pi_\beta}\left[\pzpphi \frac{\partial\log w}{\partial \vz} \grep \right] - \change{\beta}\mathbb{E}_{\pi_\beta}\left[\pzpphi \pgpz \right]\\
= &\change{\beta (1 - \beta) }\left(\mathbb{E}_{\pi_\beta}\left[\pzpphi \frac{\partial\log w}{\partial \vz} \grep \right] - \mathbb{E}_{\pibeta} \left[\pzpphi\frac{\partial\log w}{\partial \vz} \right] \mathbb{E}_{\pibeta} \left[f(\vz)\right]\right) - \beta\mathbb{E}_{\pi_\beta}\left[\pzpphi \pgpz \right] \\
= &\beta (1 - \beta) \change{\left(\text{Cov}_{\pi_{\beta}}  \left[ \grep,  \frac{\partial \vz}{\partial \phi} \frac{\partial \log w }{\partial \vz} \right]\right)} - \beta\mathbb{E}_{\pi_\beta}\left[\pzpphi \pgpz \right]\label{eq:cov_sub}.
\end{align}
By plugging \cref{eq:cov_sub} back into \cref{eq:pre_cov_sub} we arrive at the reparameterized gradient for general $\g$ \cref{eq:reparam_generic}.
\begin{align}
\frac{d }{d \phi} \mathbb{E}_{\pibeta}\left[ \grep \right] &= \change{ \beta (1 - \beta) \left(\text{Cov}_{\pi_{\beta}}  \left[ \grep,  \frac{\partial \vz}{\partial \phi} \frac{\partial \log w }{\partial \vz} \right]\right) - \beta\mathbb{E}_{\pi_\beta}\left[\pzpphi \pgpz \right]} + \mathbb{E}_{\pibeta} \left[ \frac{d }{d \phi} f(\vz) \right]\\
&= \mathbb{E}_{\pibeta} \left[\change{\left(\frac{d }{d \phi} f(\vz) \right) - \beta \left(\frac{\partial \vz}{\partial \phi} \frac{\partial f(\vz) }{\partial \vz} \right)} \right] + \beta (1-\beta) \text{Cov}_{\pi_{\beta}}  \left[  f(\vz),  \frac{\partial \vz}{\partial \phi} \frac{\partial \log w }{\partial \vz} \right]. \label{eq:last}
\end{align}

Finally, to optimize the \gls{TVO} integrand, we can substitute $f(\vz) = \log w$ for various terms in \cref{eq:last}.  We then use \cref{eq:corr1} to apply the reparameterization trick within the total derivative in the first term.
\begin{align}
    \frac{d }{d \phi} \mathbb{E}_{\pibeta}\left[ \log w \right] &= \mathbb{E}_{\pibeta} \left[ \left(\frac{d }{d \phi} \change{\log w} \right) - \beta \left(\frac{\partial \vz}{\partial \phi} \frac{\partial \change{\log w} }{\partial \vz} \right) \right] + \beta (1-\beta) \text{Cov}_{\pi_{\beta}}  \left[  \change{\log w},  \frac{\partial \vz}{\partial \phi} \frac{\partial \log w }{\partial \vz} \right] \\
    &= \mathbb{E}_{\pibeta} \bigg[ \underbrace{\change{(1 - \beta)\left(\frac{\partial \vz}{\partial \phi} \frac{\partial \log w }{\partial \vz} \right)}}_{\text{\cref{eq:corr1}}} - \beta \left(\frac{\partial \vz}{\partial \phi} \frac{\partial \log w }{\partial \vz} \right) \bigg] + \beta (1-\beta) \text{Cov}_{\pi_{\beta}}  \left[ \log w,  \frac{\partial \vz}{\partial \phi} \frac{\partial \log w }{\partial \vz} \right] \\
    &= \change{(1 - 2\beta)}\mathbb{E}_{\pibeta} \bigg[ \frac{\partial \vz}{\partial \phi} \frac{\partial \log w }{\partial \vz}  \bigg] + \beta (1-\beta) \text{Cov}_{\pi_{\beta}}  \left[ \log w,  \frac{\partial \vz}{\partial \phi} \frac{\partial \log w }{\partial \vz} \right]
\end{align}
This establishes \cref{eq:reparam_tvo} and is the expression that we use to optimize the \gls{TVO} with reparameterization in the main text.
\end{proof}

\subsection{\textsc{reparam} / \textsc{reinforce} Equivalence for $\pi_{\beta}$} 
It is well known \cite{tucker2018doubly} that the reparameterization trick and $\REINFORCE$ estimator are equivalent for expectations under $\qzx$, which allows us to trade high variance $\REINFORCE$ gradients for reparameterization gradients which directly consider derivatives of the function $\g$.
\begin{align}
    \Eq \left[ \grep \frac{\partial}{\partial \phi} \log \qzx \right] &= \Eeps \left[\pzpphi \pgpz \right]. \label{eq:appendix/grad/reparam_general}
\end{align}

We use this equivalence to show a similar result for expectations under $\pi_{\beta}$, which we will then use in the proofs of \cref{eq:lemma1} in \ref{sec:proof_lemma_1} and \cref{lemma2} in \ref{sec:proof_lemma_2}.

\begin{lemma}\label{lemma0}
    Let the same conditions hold as in \cref{eq:lemma1}. Then
\begin{align}
\mathbb{E}_{\pibeta} \left[ \grep \frac{\partial}{\partial \phi} \log \qzx \right] = \mathbb{E}_{\pi_\beta}\left[\pzpphi \left(\pgpz + \beta\grep \frac{\partial\log w}{\partial \vz}\right)\right].
\end{align}
\end{lemma}

\begin{proof}
    \begin{align}
        \mathbb{E}_{\pibeta} \left[ \grep \frac{\partial}{\partial \phi} \log \qzx \right] &= \textcolor{blue}{\frac{1}{Z_\beta}}\mathbb{E}_{\change{\qzx}} \left[\change{w^\beta} \grep \frac{\partial}{\partial \phi} \log \qzx \right] & &\text{Using \cref{eq:appendix/grad/reparam}}\\
        &= \frac{1}{Z_\beta}\mathbb{E}_{\change{\epsilon}}\left[\change{\pzpphi \frac{\partial(w^\beta \grep) }{\partial \vz}} \right]& &\text{Using \cref{eq:appendix/grad/reparam_general}}& &\\
        &= \frac{1}{Z_\beta}\mathbb{E}_{\epsilon}\left[\pzpphi \left(\change{w^\beta \pgpz + \grep\frac{\partial w^\beta}{\partial \vz}}\right)\right] & &\\
        &= \frac{1}{Z_\beta}\mathbb{E}_{\epsilon}\left[\pzpphi \left(w^\beta \pgpz + \grep \change{\beta w^\beta \frac{\partial\log w}{\partial \vz}}\right)\right] & &\\
        &= \frac{1}{Z_\beta}\mathbb{E}_{\epsilon}\left[\change{w^\beta}\pzpphi \left(\pgpz + \grep \beta \frac{\partial\log w}{\partial \vz}\right)\right] & &\\
        &= \mathbb{E}_{\textcolor{blue}{\pi_\beta}}\left[\pzpphi \left(\pgpz + \beta\grep \frac{\partial\log w}{\partial \vz}\right)\right]& &\text{Using \cref{eq:appendix/grad/reparam}}
    \end{align}

\end{proof}

\subsection{Proof of \cref{eq:lemma1}}\label{sec:proof_lemma_1}
    \begin{align}
    \mathbb{E}_{\pibeta} \left[ \grep \frac{d }{d \phi} \log w  \right] = \mathbb{E}_{\pi_\beta}\left[\pzpphi \left((1 - \beta)\grep \frac{\partial\log w}{\partial \vz} - \pgpz \right)\right].
\end{align}
\begin{proof} Using the fact that $\partial_{\phi} \log w = -\partial_{\phi} \log \qzx$,
\begin{align}
    \mathbb{E}_{\pibeta} \left[ \grep \frac{d }{d \phi} \log w  \right] &= \mathbb{E}_{\pibeta} \left[ \grep \left( \change{\frac{\partial \log w}{\partial \phi} + \pzpphi \frac{\partial \log w}{\partial \phi}}  \right)  \right] & &\\
    &= \mathbb{E}_{\pibeta} \left[ \grep \left( \change{-\frac{\partial \log \qzx}{\partial \phi}} + \pzpphi \frac{\partial \log w}{\partial \phi}  \right)  \right]& &\\
    &= \change{-}\mathbb{E}_{\pibeta} \left[ \left(\grep \frac{\partial \log \qzx}{\partial \phi}\right) \change{-} \left(\change{\grep} \pzpphi \frac{\partial \log w}{\partial \phi}  \right)  \right]& &\\
    &= -\mathbb{E}_{\pibeta} \left[ \change{\pzpphi \left(\pgpz + \beta\grep \frac{\partial\log w}{\partial \vz}\right)} - \left(\grep \pzpphi \frac{\partial \log w}{\partial \phi}  \right)  \right] & &\text{Using Lemma \ref{lemma0}}\\
    &= -\mathbb{E}_{\pibeta} \left[ \change{\pzpphi}\left( \pgpz + \beta \grep \frac{\partial\log w}{\partial \vz} - \change{\grep \frac{\partial \log w}{\partial \phi}}  \right)  \right]& &\\
    &= -\mathbb{E}_{\pibeta} \left[ \pzpphi\left( \pgpz + \change{(\beta - 1)} \grep \frac{\partial\log w}{\partial \vz} \right)  \right]& &\\
    &= \phantom{-}\mathbb{E}_{\pi_\beta}\left[\pzpphi \left(\change{(1 - \beta)}\grep \frac{\partial\log w}{\partial \vz} \change{-} \pgpz \right)\right]& &.
\end{align}
\end{proof}

\subsection{Proof of \cref{lemma2}}\label{sec:proof_lemma_2}
\begin{align}
    \frac{d }{d \phi} Z_{\beta} = \beta (1 - \beta)\mathbb{E}_{\epsilon} \left[w^{\beta} \pzpphi\frac{\partial\log w}{\partial \vz} \right].
\end{align}
\begin{proof}
    Noting that we can use reparameterization inside the integral $Z_{\beta} = \int \qzx^{1-\beta} p_{\theta}(\x,\vz)^{\beta} dz = \mathbb{E}_{\qphi}[w^{\beta}] =  \mathbb{E}_{\epsilon}[w^{\beta}]$, we obtain
\begin{align}
    \frac{d }{d \phi} Z_{\beta} &= \frac{d }{d \phi}\mathbb{E}_{\epsilon}[w^{\beta}]\\
     &= \mathbb{E}_{\epsilon}\left[\change{\beta w^{\beta} \frac{d }{d \phi} \log w}\right]\\
     &= \change{\beta Z_{\beta}} \mathbb{E}_{\change{\pibeta}}\left[\frac{d }{d \phi} \log w\right]\\
     &= \beta \change{(1 - \beta)}Z_{\beta} \mathbb{E}_{\pibeta} \left[ \pzpphi\frac{\partial\log w}{\partial \vz} \right]~\quad \text{Using \cref{eq:corr1}}\\
     &= \beta (1 - \beta)\mathbb{E}_{\change{\epsilon}} \left[\change{w^{\beta}}\pzpphi\frac{\partial\log w}{\partial \vz} \right]
\end{align}
\end{proof}


\end{document}